\documentclass{article}

\usepackage{iftex}
\ifPDFTeX
\else
  \ifLuaTeX
  \else
    \ifXeTeX
    \else
      \PassOptionsToPackage{dvipdfmx}{graphicx}
      \PassOptionsToPackage{dvipdfmx}{xcolor}
      \PassOptionsToPackage{dvipdfmx}{hyperref}
    \fi
  \fi
\fi

\usepackage{microtype}
\usepackage{graphicx}
\usepackage{subcaption}
\usepackage{booktabs}
\usepackage{hyperref}
\usepackage{amsthm,amsmath,amssymb,amsfonts}
\usepackage{algorithmic}
\usepackage{textcomp}
\usepackage{xcolor}
\usepackage{mathtools,mathrsfs}
\usepackage{bm}
\usepackage{enumitem}
\usepackage{dsfont}
\usepackage[capitalize,noabbrev]{cleveref}

\usepackage[accepted]{icml2026}

\newcommand{\R}{\mathbb{R}}

\DeclareMathOperator{\Hom}{Hom}

\DeclareMathOperator{\Aut}{Aut}
\DeclareMathOperator{\Stab}{Stab}
\DeclareMathOperator{\GL}{GL}

\newcommand{\id}{\mathrm{id}}
\newcommand{\calO}{\mathcal O}
\newcommand{\calQ}{\mathcal Q}

\theoremstyle{plain}
\newtheorem{theorem}{Theorem}[section]
\newtheorem{proposition}[theorem]{Proposition}
\newtheorem{lemma}[theorem]{Lemma}
\newtheorem{corollary}[theorem]{Corollary}
\theoremstyle{definition}
\newtheorem{definition}[theorem]{Definition}

\theoremstyle{remark}
\newtheorem{remark}[theorem]{Remark}

\usepackage[disable,textsize=tiny]{todonotes}

\icmltitlerunning{Foundations of Equivariant Deep Learning}

\begin{document}

\twocolumn[
  \icmltitle{Foundations of Equivariant Deep Learning:\\ Unifying Graph and Sheaf Neural Networks}

  \icmlsetsymbol{equal}{*}

  \begin{icmlauthorlist}
    \icmlauthor{Yoshihiro Maruyama}{kyoto}
  \end{icmlauthorlist}

  \icmlaffiliation{kyoto}{Department of Mathematical and Information Sciences, Kyoto University, Kyoto, Japan}
  
  \icmlcorrespondingauthor{Yoshihiro Maruyama}{maruyama@i.h.kyoto-u.ac.jp}

  \icmlkeywords{Geometric Deep Learning, Topological Deep Learning, Categorical Deep Learning, Equivariant Universal Approximation, Equivariant Vector Bundle, Sheaf Neural Net, Category-Equivariant Neural Network, Categorical Equivariant Deep Learning, Categorical Symmetry}

  \vskip 0.3in
]

\printAffiliationsAndNotice{}

\begin{abstract}
Symmetry is everywhere in nature and society. Geometric deep learning builds architectures respecting group symmetries, whereas topological deep learning organizes computation through cells, incidence relations, and local-to-global structure. In this paper, we extend geometric deep learning beyond simple group actions and unify it with topological deep learning. Specifically, we develop order-equivariant neural networks (OENN), which generalize standard graph message passing and sheaf neural networks via the theory of equivariant vector bundles over face posets (or face categories). We (i) characterize all linear order-equivariant maps, (ii) build OENN layers, and (iii) prove universal approximation theorems (UATs) for continuous order-equivariant maps, which are new results even when restricted to sheaf neural networks. We illustrate the framework on graph and sheaf models. Our results can also be seen as extending the known UAT for graph neural networks to a more general setting that subsumes sheaf neural networks as well. In the appendix, we show that OENN can be connected, via the action groupoid Grothendieck construction, to CENN (category-equivariant neural network), which gives the categorical general form of equivariant neural networks, allowing us to leverage categorical symmetry in data (e.g., non-invertible symmetries on multiple objects with compositional relations on those symmetries).
\end{abstract}

\section{Introduction}
\label{sec:intro}

{\bf Motivation.}
Symmetry has long been recognized as a powerful inductive bias in deep learning.  
Group-equivariant architectures explicitly exploit the symmetries of an index set to share parameters, improve sample efficiency, and ensure consistent behavior under transformations of the input domain~\cite{CohenWelling16,CohenWelling17Steer,WeilerCesa19,Fuchs20,Satorras21}.  
Classical convolutional networks are equivariant to translations, while group-equivariant CNNs generalize this idea to arbitrary compact groups.  
More recently, this principle has been extended to permutation-invariant and permutation-equivariant models such as DeepSets~\cite{Zaheer17} and graph neural networks GNNs~\cite{Gilmer17,Xu19,Maron19}, which achieve equivariance with respect to the automorphism group of a graph.  
However, many data domains are not merely sets or graphs but possess richer hierarchical or incidence structures.  
Examples include face-posets of graphs, simplicial or CW complexes, and general partially ordered collections of cells or regions.  
Such domains are naturally indexed by posets, and their symmetries form automorphism groups $G=\Aut(P,\le)$ that may involve both permutation and structural (non-invertible or higher-order) relations.  
This observation motivates the development of \emph{order-equivariant neural networks} (OENN) that respect the combinatorial structure of a poset rather than a pure group action.

{\bf Background and context.}
The study of equivariance has unified diverse neural architectures under the umbrella of geometric deep learning~\cite{Bronstein21,Shuman13,Bruna14,Defferrard16}.
Group- and permutation-equivariant networks~\cite{CohenWelling16,Zaheer17,Gilmer17,Xu19,Maron19} rely on the algebraic structure of groups acting on index sets, while higher-order analogs exploit topological constructions such as simplicial, cellular, or combinatorial complexes~\cite{Ebli21,Bodnar22,HansenGebhart20,Barbero22}.  
This includes combinatorial-complex neural networks for topological deep learning~\cite{Hajij22}; in the equivariant line, $E(n)$-equivariant topological neural networks incorporate Euclidean symmetry into topological message passing~\cite{Battiloro24}.
Sheaf theory~\cite{Curry14} provides a general language for modeling local-to-global dependencies on such structured domains, and recent work on neural sheaf architectures~\cite{Hajij25,Bamberger24,Bodnar22,Barbero22,HansenGebhart20} shows that combining algebraic-topological tools with deep learning yields strong theoretical and empirical advantages.  
Our work generalizes these frameworks further: from group- or permutation-equivariant settings to order-equivariant ones, thereby covering both classical GNNs and sheaf-based models as special cases.

{\bf Contributions.}
We make three contributions:
\begin{itemize}
  \item[(1)] We give a unified definition of order-equivariant maps between poset-indexed feature bundles and provide a complete orbit-wise characterization of all linear OENN layers via the transporter law and stabilizer intertwiners. This generalizes the block-tying rules of permutation-equivariant networks~\cite{Maron19} to arbitrary posets.
  \item[(2)] We turn this linear theory into a nonlinear architecture.  OENN layers combine orbital affine maps, equivariant biases, pointwise Reynolds blocks, and pair-orbit aggregation, so nonlinearities remain equivariant even when stabilizers act nontrivially on fibers.  This yields a hierarchy containing ordinary relation-message-passing OENNs, source-labeled pair-state OENNs, and full OENNs, and it recovers DeepSets, fixed-graph message passing, vertex-edge incidence updates, and cellular or simplicial sheaf layers as special or further-tied cases~\cite{Zaheer17,Gilmer17,Xu19,Curry14,Bodnar22,Barbero22}.
  \item[(3)] We prove \emph{equivariant universal approximation theorems} (UATs) for OENNs.  The full OENN class is dense in the continuous order-equivariant maps on compact $G$-invariant sets.  We also separate this result from ordinary bounded-depth message passing: standard anonymous aggregation is not universal in general, while pair-state local OENNs compile the global broadcast in directed diameter depth and give a diameter-sharp local universality theorem, with a cover-local corollary for connected Hasse graphs.
\end{itemize}
We show in the appendix that OENN can be connected, via the action groupoid Grothendieck construction, to CENN (category-equivariant neural network), which gives the categorical general form of equivariant neural networks, allowing us to leverage categorical symmetry in data (e.g., non-invertible symmetries on multiple objects with compositional relations on those symmetries).

{\bf Relation to prior work.}
Our results can be viewed as a natural extension of the known UATs for permutation- and group-equivariant architectures~\cite{Zaheer17,Maron19,Xu19,CohenWelling16,Satorras21}, bringing them under a single formulation that also
subsumes sheaf neural networks~\cite{Curry14,HansenGebhart20,Barbero22,Bodnar22}.  
Combinatorial-complex neural networks assign features to cells of a combinatorial complex and update them through incidence, adjacency, or coadjacency relations~\cite{Hajij22}; on a fixed finite complex, such relations are $G$-stable pair relations on the face poset whenever the automorphism group $G\leq\Aut(P,\leq)$ preserves them, so automorphism-equivariant CCNN message-passing layers are instances of relation-message-passing OENNs.
From a geometric perspective, OENN generalizes the principles of geometric deep learning~\cite{Bronstein21} to domains indexed by partially ordered sets, bridging discrete symmetry, incidence structure, and topological learning within one unified formalism.
The framework of categorical equivariant deep learning develops category-equivariant neural networks in general form, providing both equivariant universality theorems and experimental performance gains~\cite{Maruyama25CENN,Maruyama25HARRep,Maruyama25HARArch,Maruyama26CatReg,Maruyama26InfinityHAR,MaruyamaYasuda26Grothendieck,NasuMaruyama26Monoid}. We clarify the precise mathematical relationships between OENN and CENN in the appendix.

{\bf Organization.}
Section~\ref{sec:oenn} introduces the formal setting of order-equivariant bundles
and the construction of OENN layers. Section~\ref{sec:uat} presents 
equivariant universal approximation theorems for various types of OENNs. 
Section~\ref{sec:examples} illustrates the framework on graphs and sheaves,
showing how OENN unifies message-passing and sheaf neural networks.
The paper concludes with a brief discussion of significance and applicability.
The appendix connects OENN with CENN. It also provides all omitted proofs.

\section{Order-Equivariant Neural Networks}
\label{sec:oenn}

\subsection{Posets, actions, and equivariant bundles}

Let $(P,\leq)$ be a finite poset and let $G$ be a finite group acting on $P$ by order automorphisms; equivalently, fix a homomorphism $G\to\Aut(P,\leq)$.  We write the action as $(\gamma,p)\mapsto\gamma p$.  The full automorphism group case is the special case $G=\Aut(P,\leq)$.  All vector spaces are finite-dimensional real vector spaces.

\begin{definition}[Equivariant bundle over a poset]
\label{def:bundle}
A $G$-equivariant vector bundle over $P$ consists of vector spaces $(V_p)_{p\in P}$ and linear isomorphisms
\[
T^V_{\gamma,p}:V_p\to V_{\gamma p}\qquad(\gamma\in G,\ p\in P)
\]
such that $T^V_{\id,p}=\id_{V_p}$ and
\[
T^V_{\gamma\eta,p}=T^V_{\gamma,\eta p}\circ T^V_{\eta,p}
\qquad(\gamma,\eta\in G,
 p\in P).
\]
The total feature space is
\[
X:=\bigoplus_{p\in P}V_p.
\]
The bundle transports induce a representation $\rho_X:G\to\GL(X)$ by
\begin{equation}
\label{eq:rhoX-def}
(\rho_X(\gamma)x)_q=T^V_{\gamma,\gamma^{-1}q}x_{\gamma^{-1}q}.
\end{equation}
A second bundle $(W_q,T^W_{\gamma,q})$ has total space $Y=\bigoplus_{q\in P}W_q$ and induced action $\rho_Y$ by the same formula.  A \emph{permutation-only} bundle is the special case in which all transports are identity maps between canonically identified fibers; then $\rho_X$ and $\rho_Y$ are block permutations.
\end{definition}

{\bf Intuition.}
An element $p\in P$ should be read as a site, cell, or component of the structured domain, and the fiber $V_p$ is the local feature space carried by that site.  A symmetry $\gamma\in G$ relabels sites by $p\mapsto\gamma p$, while the transport $T^V_{\gamma,p}:V_p\to V_{\gamma p}$ specifies how feature coordinates are carried along this relabeling.  The cocycle identity says that transporting along $\eta$ and then along $\gamma$ is the same as transporting along $\gamma\eta$, so the direct sum of all fibers inherits a genuine linear $G$-representation.  The order relation on $P$ is not itself a fiber map; rather, it determines natural $G$-stable relations, masks, and pair-orbits used by the layers below.  Sheaf restriction maps are additional structure on such fibers and are incorporated in Section~\ref{subsec:sheaves}.

\begin{definition}[Order-equivariant map]
\label{def:oe-map}
Let $K\subseteq X$ be $G$-invariant. A continuous map $F:K\to Y$ is \emph{order-equivariant} if
\begin{equation}
\label{eq:equivariance}
F(\rho_X(\gamma)x)=\rho_Y(\gamma)F(x)
\qquad(\gamma\in G,\ x\in K).
\end{equation}
We write $C_G(K,Y)$ for the space of continuous order-equivariant maps $K\to Y$.  When $K=X$, we also write $C_G(X,Y)$.  
\end{definition}

\subsection{Linear and affine order-equivariant maps}

Any linear map $L:X\to Y$ has a block-kernel representation
\begin{equation}
\label{eq:block-kernel}
(Lx)_q=\sum_{p\in P}K(q,p)x_p,
\qquad K(q,p)\in\Hom(V_p,W_q).
\end{equation}

\begin{proposition}[Transporter law]
\label{prop:transporter-law}
A linear map $L:X\to Y$ with kernels $K(q,p)$ is order-equivariant if and only if
\begin{equation}
\label{eq:transporter-law}
K(\gamma q,\gamma p)T^V_{\gamma,p}=T^W_{\gamma,q}K(q,p)
\qquad(\gamma\in G,
 p,q\in P).
\end{equation}
\end{proposition}

Let $G$ act diagonally on $P\times P$ by $\gamma(q,p)=(\gamma q,\gamma p)$.  For a pair-orbit $\calO\in(P\times P)/G$, fix a representative $(q_{\calO},p_{\calO})$ and set
\[
H_{\calO}:=\{\eta\in G:\eta q_{\calO}=q_{\calO},\ \eta p_{\calO}=p_{\calO}\}.
\]

\begin{proposition}[Orbit parametrization]
\label{prop:orbit-param}
For each pair-orbit $\calO$, choose
\[
A_{\calO}\in\Hom(V_{p_{\calO}},W_{q_{\calO}})
\]
satisfying
\begin{equation}
\label{eq:stabilizer-intertwiner}
T^W_{\eta,q_{\calO}}A_{\calO}=A_{\calO}T^V_{\eta,p_{\calO}}
\qquad(\eta\in H_{\calO}).
\end{equation}
Define
\begin{equation}
\label{eq:orbit-transport}
K(q,p):=T^W_{\gamma,q_{\calO}}A_{\calO}\bigl(T^V_{\gamma,p_{\calO}}\bigr)^{-1}
\end{equation}
whenever $(q,p)=\gamma(q_{\calO},p_{\calO})$.  Then $K$ is well-defined and satisfies the transporter law.  Conversely, every linear order-equivariant map arises uniquely in this way.  Hence
\[
\begin{split}
\dim\Hom_G(X,Y)
&=\sum_{\calO\in(P\times P)/G}\\
&\quad \dim\Hom_{H_{\calO}}(V_{p_{\calO}},W_{q_{\calO}}).
\end{split}
\]
\end{proposition}

\begin{corollary}[Permutation-only case]
\label{cor:perm-only}
If the input and output bundles are permutation-only, so that source fibers and target fibers in each site-orbit are canonically identified and all transports are identities under these identifications, then the stabilizer constraint is vacuous and, after the same canonical identifications, $K(q,p)$ is constant on each pair-orbit.
\end{corollary}

{\bf Auxiliary finite $G$-sets.}
We will also use the same linear theory for auxiliary finite $G$-sets.  If $S$ and $T$ are finite $G$-sets, $X=\bigoplus_{s\in S}V_s$ and $Y=\bigoplus_{t\in T}W_t$ are equivariant bundles over them, and
\[
(Lx)_t=\sum_{s\in S}K(t,s)x_s,
\qquad K(t,s)\in\Hom(V_s,W_t),
\]
then $L$ is $G$-equivariant if and only if
\[
K(\gamma t,\gamma s)T^V_{\gamma,s}=T^W_{\gamma,t}K(t,s)
\qquad(\gamma\in G,\ s\in S,\ t\in T).
\]
The orbit parametrization is identical with $P\times P$ replaced by $T\times S$.  Thus an orbital affine layer below may be used between bundles over such auxiliary $G$-sets, and a mask is equivariant exactly when its support is a union of $G$-orbits in the relevant target-source product.

\begin{definition}[Orbital affine layer]
\label{def:orbital-affine}
An \emph{orbital affine layer} is a map $A:X\to Y$ of the form $A(x)=Lx+b$, where $L$ is a linear order-equivariant map and $b\in Y$ is $G$-fixed: $\rho_Y(\gamma)b=b$ for all $\gamma\in G$.  Equivalently,
\[
b_{\gamma q}=T^W_{\gamma,q}b_q.
\]
Thus $b$ is specified by choosing, for each site-orbit representative $q_0$, a vector $b_{q_0}\in W_{q_0}^{G_{q_0}}$.
\end{definition}

Without biases, the UAT fails for common activations such as ReLU and $\tanh$ even for a one-point poset and trivial group, since no-bias networks with $\sigma(0)=0$ vanish at the origin.  We include equivariant biases to match the usual affine-map formulation of universal approximation.

\subsection{Reynolds blocks and OENN layers}

An ordinary MLP is a finite composition of affine maps and coordinatewise application of a fixed scalar activation $\sigma:\R\to\R$.

\begin{definition}[Reynolds block]
\label{def:reynolds-block}
Let $H$ be a finite group acting linearly on finite-dimensional spaces $U$ and $V$ by representations $T^U$ and $T^V$.  For an ordinary MLP $\Psi:U\to V$, define
\begin{equation}
\label{eq:eq-reynolds-block}
\mathcal R_H^{\mathrm{eq}}[\Psi](u):=\frac1{|H|}\sum_{h\in H}(T^V_h)^{-1}\Psi(T^U_hu).
\end{equation}
This is called a Reynolds block.  If $V=\R^m$ has the trivial $H$-action, the corresponding invariant block is
\begin{equation}
\label{eq:inv-reynolds-block}
\mathcal R_H^{\mathrm{inv}}[\Psi](u):=\frac1{|H|}\sum_{h\in H}\Psi((T^U_h)^{-1}u).
\end{equation}
\end{definition}

\begin{lemma}[Equivariance and branch realization]
\label{lem:reynolds-realization}
The map \eqref{eq:eq-reynolds-block} is $H$-equivariant.  Moreover, it is realizable by equivariant affine maps and coordinatewise activations on hidden branch spaces whose $H$-action is by permutation of branches.
\end{lemma}

\begin{definition}[Pointwise Reynolds layer]
\label{def:pointwise-reynolds-layer}
Let $S$ be a finite $G$-set, and let $E=\bigoplus_{s\in S}E_s$ and $F=\bigoplus_{s\in S}F_s$ be equivariant bundles over $S$.  For each representative $s_0$ of a $G$-orbit in $S$, choose an ordinary MLP $\Psi_{s_0}:E_{s_0}\to F_{s_0}$ and form the Reynolds block
\[
\psi_{s_0}:=\mathcal R^{\mathrm{eq}}_{G_{s_0}}[\Psi_{s_0}],
\]
where $G_{s_0}$ acts on $E_{s_0}$ and $F_{s_0}$ by the stabilizer transports.  For $s=\gamma s_0$, define
\begin{equation}
\label{eq:pointwise-transport}
\psi_s(e):=T^F_{\gamma,s_0}\psi_{s_0}\bigl((T^E_{\gamma,s_0})^{-1}e\bigr),
\qquad e\in E_s.
\end{equation}
The map $N:E\to F$ given by $(Nz)_s=\psi_s(z_s)$ is a \emph{pointwise Reynolds layer}.
\end{definition}

The definition is independent of the choice of $\gamma$ because $\psi_{s_0}$ is $G_{s_0}$-equivariant.  A direct calculation gives $N\rho_E(\tau)=\rho_F(\tau)N$ for every $\tau\in G$.

\begin{definition}[OENN]
\label{def:oenn}
An \emph{order-equivariant neural network} (OENN) on the poset $P$ is any finite composition of orbital affine layers and pointwise Reynolds layers, allowing intermediate equivariant bundles and finite parallel concatenations.  Intermediate bundles may be indexed by $P$ or by auxiliary finite $G$-sets functorially built from $P$, such as $P\times P$ for pair-lift constructions.  For auxiliary source and target bundles over finite $G$-sets $S$ and $T$,
``orbital affine'' means the same transporter-law affine layer with
$P\times P$ replaced by $T\times S$. 
Local or masked orbital affine layers are allowed when their supports are unions of pair-orbits in the corresponding target-source product.
\end{definition}

When no locality mask is imposed on the orbital affine layers, we call the resulting architecture class the \emph{full} OENN class.  
Scalar activations are applied inside ordinary MLP branches, and equivariance for arbitrary stabilizer representations is enforced by Reynolds averaging.  
In the special case of permutation-only bundles, pointwise Reynolds layers reduce to the usual shared coordinatewise MLP layers.

\begin{definition}[$R$-local layers, pair-state locality, and covers]
\label{def:R-local-oenn}
Let $R\subseteq P\times P$ be a $G$-invariant relation containing the diagonal.  An orbital affine layer $A:E\to F$ between bundles indexed by $P$ is called \emph{$R$-local} if its kernel blocks satisfy
\[
K(q,p)=0
\qquad\text{unless }(q,p)\in R.
\]
Pointwise Reynolds layers are diagonal in the site index and hence are $R$-local because $\Delta_P\subseteq R$.

The \emph{pair-state} or \emph{source-labeled} $R$-local class, denoted
$\mathrm{OENN}^{R\text{-pair}}_\sigma(X,Y)$, uses hidden bundles whose carrier site $q$ may contain source-indexed slots
\(
E_q=\bigoplus_{p\in P}E_{q,p}.
\)
The group action sends the $p$-source slot over $q$ to the $\gamma p$-source slot over $\gamma q$.  Equivalently, the inter-site hidden state is a bundle over the pair-state set $P\times P$ with diagonal $G$-action.  An inter-site orbital affine layer between pair-state bundles is called \emph{pair-state $R$-local} if its kernel blocks satisfy the source-preserving support condition
\[
K((q,p),(r,s))=0
\qquad\text{unless } s=p \text{ and } (q,r)\in R.
\]
Thus the carrier coordinate propagates locally along $R$, while the source label is transported equivariantly and is not mixed during inter-site propagation.  The class $\mathrm{OENN}^{R\text{-pair}}_\sigma(X,Y)$ is generated by $R$-local orbital affine layers on $P$, pair-state $R$-local orbital affine layers, pointwise Reynolds layers, carrier-diagonal read/write maps between $P$-indexed and $P\times P$-indexed bundles, and finite parallel concatenations.  Carrier-diagonal maps and pointwise Reynolds layers may act on the whole carrier fiber $E_q=\bigoplus_{p\in P}E_{q,p}$, but they do not propagate information between different carrier sites.  Thus pair-state OENNs preserve the communication pattern of local message passing while retaining the identity of the source site.  This is a higher-order local class (rather than the ordinary anonymous-message-passing class).

Write $q\prec p$ when $q<p$ and no element $r\in P$ satisfies $q<r<p$.  The cover-local, order-respecting relation is
\[
R_{\mathrm{cov}}
:=
\Delta_P
\cup \{(q,p):q\prec p\}
\cup \{(q,p):p\prec q\}.
\]
Thus an $R_{\mathrm{cov}}$-local layer communicates only along self-loops and one-step up/down Hasse-cover incidences.
\end{definition}

\subsection{Pair-orbit aggregation}
\label{subsec:pair-orbit-aggregation}

We now give a practical aggregation form that is well-defined for arbitrary stabilizer actions.  For fixed $q\in P$, put $G_q:=\Stab_G(q)$ and let
\[
P/G_q=\{\Omega_1(q),\ldots,\Omega_{R(q)}(q)\}
\]
be the relative-position classes of source sites as seen from $q$.  Pair-orbits $\calO\in(P\times P)/G$ meeting $\{q\}\times P$ are canonically in bijection with these classes by
\[
\calO\longmapsto \Omega_{\calO}(q):=\{p\in P:(q,p)\in\calO\}.
\]
In this sense, pair-orbit aggregation is an order-aggregating block: the summaries $S_{\calO}(q;x)$ below aggregate over source sites of a fixed relative-position type, while the global pair-orbit notation avoids choosing $q$-dependent representatives.

Let $\calO\in(P\times P)/G$ be a pair-orbit with representative $(q_{\calO},p_{\calO})$ and stabilizer $H_{\calO}$.  Choose an $H_{\calO}$-invariant Reynolds block
\(
\phi_{\calO}:V_{p_{\calO}}\to\R^{m_{\calO}}.
\)
For $(q,p)=\gamma(q_{\calO},p_{\calO})$, define
\begin{equation}
\label{eq:transported-phi}
\phi_{q,p}(v):=\phi_{\calO}\bigl((T^V_{\gamma,p_{\calO}})^{-1}v\bigr),
\qquad v\in V_p.
\end{equation}
This is independent of $\gamma$, because changing $\gamma$ changes it by an element of $H_{\calO}$ and $\phi_{\calO}$ is $H_{\calO}$-invariant.  Define the summary
\begin{equation}
\label{eq:pair-summary}
S_{\calO}(q;x):=\sum_{\substack{p\in P\\(q,p)\in\calO}}\phi_{q,p}(x_p)\in\R^{m_{\calO}},
\end{equation}
with the empty sum interpreted as zero.

\begin{lemma}[Equivariance of pair summaries]
\label{lem:pair-summary-equiv}
For all $\tau\in G$,
\[
S_{\calO}(\tau q;\rho_X(\tau)x)=S_{\calO}(q;x).
\]
In particular, $S_{\calO}(q;\cdot)$ is $G_q$-invariant.
\end{lemma}

More generally, the same construction can be grouped over any $G$-invariant relation $R\subseteq P\times P$.  Since $R$ is a union of pair-orbits, one may either feed the separate summaries $\{S_{\calO}:\calO\subset R\}$ to the readout, or, when the summary dimensions agree and the corresponding encoders are intentionally tied, use the coarsened relation summary
\[
S_R(q;x):=\sum_{\calO\subset R}S_{\calO}(q;x)
=\sum_{\substack{p\in P\\(q,p)\in R}}\phi^R_{q,p}(x_p).
\]
Here $\phi^R_{q,p}$ denotes the tied transported encoder on the pair-orbit containing $(q,p)$.  Lemma~\ref{lem:pair-summary-equiv} implies $S_R(\tau q;\rho_X(\tau)x)=S_R(q;x)$, so relation-level aggregation is an equivariant parameter-tied coarsening of the orbit-by-orbit construction.

For each site-orbit representative $q_0$, let $\mathcal A(q_0)$ be the set of pair-orbits $\calO$ for which some pair $(q_0,p)$ lies in $\calO$.  Choose a $G_{q_0}$-equivariant Reynolds block
\[
\psi_{q_0}:V_{q_0}\oplus\bigoplus_{\calO\in\mathcal A(q_0)}\R^{m_{\calO}}\to W_{q_0},
\]
where $G_{q_0}$ acts on $V_{q_0}$ by the fiber transport and trivially on the summary coordinates.  For $q=\gamma q_0$, define
\begin{equation}
\label{eq:pair-aggregation-layer}
F(x)_q:=T^W_{\gamma,q_0}\psi_{q_0}\left((T^V_{\gamma,q_0})^{-1}x_q,
\{S_{\calO}(q;x)\}_{\calO\in\mathcal A(q_0)}\right).
\end{equation}

\begin{proposition}[Pair-orbit aggregation is an OENN layer]
\label{prop:pair-aggregation-oenn}
The map $F$ in \eqref{eq:pair-aggregation-layer} is order-equivariant and is realizable as an OENN layer stack.
\end{proposition}

\begin{definition}[Relation-message-passing OENN layer]
\label{def:relation-mp-oenn}
Let $R\subseteq P\times P$ be a $G$-invariant relation containing the diagonal.  A \emph{relation-message-passing OENN layer} is an update obtained from the pair-orbit aggregation construction by using only pair-orbits $\calO\subseteq R$:
\begin{equation}
\label{eq:relation-mp-oenn}
h'_q=\theta_q\Bigl(h_q,\{S_{\calO}(q;h)\}_{\calO\subseteq R}\Bigr).
\end{equation}
Here the encoders defining $S_{\calO}$ are transported invariant Reynolds blocks as in \eqref{eq:transported-phi}-\eqref{eq:pair-summary}, and the readouts $\theta_q$ are transported from stabilizer-equivariant Reynolds blocks on site-orbit representatives.  If several pair-orbits inside $R$ are intentionally tied and summed before the readout, \eqref{eq:relation-mp-oenn} gives the usual relation-level aggregation.  We write $\mathrm{OENN}^{R\text{-mp}}_\sigma$ for finite compositions of such layers.
\end{definition}

\begin{proposition}[Message-passing-compatible updates]
\label{prop:mp-compatible}
Every relation-message-passing OENN layer is an order-equivariant OENN stack.  Conversely, any continuous update of the form \eqref{eq:relation-mp-oenn} whose encoders are stabilizer-invariant and whose readouts are stabilizer-equivariant can be uniformly approximated on compact invariant sets by relation-message-passing OENN layers with any continuous non-polynomial activation $\sigma$.
\end{proposition}

\begin{proposition}[Architecture hierarchy]
\label{prop:oenn-architecture-hierarchy}
For every $G$-invariant relation $R\subseteq P\times P$ containing the diagonal and every activation $\sigma$ used in the OENN layers,
\[
\mathrm{OENN}^{R\text{-mp}}_\sigma
\subseteq
\mathrm{OENN}^{R\text{-pair}}_\sigma
\subseteq
\mathrm{OENN}^{\mathrm{full}}_\sigma.
\]
\end{proposition}

The first class is the analog of ordinary local message passing: messages from neighbors in $R$ are aggregated without a persistent global source label.  The pair-state class is the universal local completion used in Section~\ref{subsec:local-uat}: it keeps the same communication graph but carries a source index through propagation.  The full class allows global orbital affine layers.

\section{Equivariant Universal Approximation}
\label{sec:uat}

Throughout this section, let $(P,\leq)$ be a finite poset, let $G$ be a finite group acting on $P$ by order automorphisms, and let $X=\bigoplus_{p\in P}V_p$ and $Y=\bigoplus_{q\in P}W_q$ be $G$-equivariant bundles with actions $\rho_X$ and $\rho_Y$.  Approximation is measured uniformly on compact $G$-invariant subsets $K\subset X$.

\subsection{Preparatory lemmas}

\begin{lemma}[Sitewise stabilizer equivariance]
\label{lem:sitewise-stabilizer}
Let $K\subseteq X$ be $G$-invariant. If $F\in C_G(K,Y)$, then for every $q\in P$ and every $\eta\in G_q:=\Stab_G(q)$,
\[
F(\rho_X(\eta)x)_q=T^W_{\eta,q}F(x)_q
\qquad(x\in K).
\]
\end{lemma}

\begin{lemma}[Equivariant density for finite groups]
\label{lem:finite-group-density}
Let $H$ be a finite group acting linearly on finite-dimensional real vector spaces $U$ and $V$ by $T^U$ and $T^V$.  Let $K\subset U$ be compact and $H$-invariant, and let $f:K\to V$ be continuous and $H$-equivariant.  If $\sigma:\R\to\R$ is continuous and non-polynomial, then for every norm $\|\cdot\|_*$ on $V$ and every $\varepsilon>0$ there is an ordinary MLP $\Psi:U\to V$ with activation $\sigma$ such that
\[
\psi(u):=\frac1{|H|}\sum_{h\in H}(T^V_h)^{-1}\Psi(T^U_hu)
\]
is $H$-equivariant and satisfies
\[
\sup_{u\in K}\|\psi(u)-f(u)\|_*<\varepsilon.
\]
Moreover, $\psi$ is a Reynolds block and hence is OENN-realizable in the sense of Lemma~\ref{lem:reynolds-realization}.
\end{lemma}

\subsection{Full universal approximation}

We first consider $\mathrm{OENN}^{\mathrm{full}}_\sigma(X,Y)$, i.e., the class of OENNs from Definition~\ref{def:oenn} in which arbitrary orbital affine layers are allowed.

\begin{theorem}[Full OENN UAT]
\label{thm:OENN-UAT}
Let $(P,\leq)$ be finite and let $G$ be a finite group acting on $P$ by order automorphisms.  Let $X=\bigoplus_{p\in P}V_p$ and $Y=\bigoplus_{q\in P}W_q$ be $G$-equivariant bundles.  Let $K\subset X$ be compact and $G$-invariant, and let $\sigma:\R\to\R$ be continuous and non-polynomial.  Then $\mathrm{OENN}^{\mathrm{full}}_\sigma(X,Y)$ is dense in $C_G(K,Y)$ in the uniform norm.  Equivalently, for every $f\in C_G(K,Y)$ and every $\varepsilon>0$, there exists a full OENN $F:X\to Y$ such that
\[
\sup_{x\in K}\|F(x)-f(x)\|<\varepsilon.
\]
\end{theorem}

Note that local masks such as cover, incidence, or one-hop graph relations are equivariant whenever they are $G$-stable, but bounded-depth local equivariance alone does not imply universal approximation.

\begin{corollary}[Permutation-only UAT]
\label{cor:perm-only-uat}
Assume that the original input and output transports are identity maps between canonically identified fibers, so that $G$ acts by block permutations.  Then the linear part of OENN specializes to affine layers whose blocks and biases are tied on $G$-orbits.  Ordinary sitewise pointwise layers reduce to shared MLPs on site-orbits, while stabilizer-equivariant readouts on auxiliary broadcast fibers are finite Reynolds averages over permutation representations.  Consequently, Theorem~\ref{thm:OENN-UAT} specializes to universal approximation of continuous $G$-equivariant maps on compact invariant subsets for finite permutation-group actions.  In full-symmetric invariant-output special cases, this is compatible with DeepSets-type sum-of-embeddings representations, while for graph-indexed features it gives the full fixed-domain orbit-tied equivariant approximation class.
\end{corollary}

\subsection{Pair-state local message-passing universality}
\label{subsec:local-uat}

The proof of Theorem~\ref{thm:OENN-UAT} uses a global broadcast layer.  We now show that the same effect can be obtained by local message passing, provided the hidden state is allowed to be \emph{pair-state}: the state $z_{q,p}$ is stored at carrier site $q$ and remembers that it originated at source site $p$.  This is the local analog of higher-order or tensorized GNNs; it keeps the same communication graph as ordinary anonymous aggregation.

For a $G$-invariant relation $R\subseteq P\times P$ containing the diagonal, let $\Gamma_R$ be the directed communication graph with vertex set $P$ and edge $p\to q$ exactly when $(q,p)\in R$.  Thus $R$ records which source sites $p$ may send information to which target sites $q$ in one carrier-local linear layer.  Write $d_R(p,q)$ for the directed distance from $p$ to $q$, with $d_R(p,q)=\infty$ if no such path exists.

\begin{theorem}[Diameter-sharp pair-state local universality]
\label{thm:diameter-sharp-local-uat}
Let $R\subseteq P\times P$ be a $G$-invariant relation containing the diagonal, and let $\Gamma_R$ and $d_R$ be as above.

(i) Lower bound. Fix $L\geq 0$.  Consider any architecture whose input feature at site $p$ is initially available only at $p$, whose inter-site layers have support contained in $R$, and whose remaining operations are pointwise nonlinearities, carrier-diagonal equivariant maps, or finite parallel concatenations.  Then the output at $q$ after at most $L$ inter-site layers depends only on input sites $p$ with $d_R(p,q)\leq L$.  Consequently, if $d_R(p,q)>L$, then even for trivial $G$, scalar fibers, and compact domain $[0,1]^P$, no such $L$-local architecture can uniformly approximate the continuous map whose $q$-coordinate is $f(x)_q=x_p$ and whose other coordinates are zero.

(ii) Diameter-depth compilation. If $\Gamma_R$ is strongly connected with directed diameter $D$, then the broadcast layer
\[
B:X\to\bigoplus_{q\in P}\widetilde X_q,
\qquad
\widetilde X_q:=X,
\qquad
(Bx)_q=x,
\]
with transport $T^{\widetilde X}_{\gamma,q}=\rho_X(\gamma)$ is exactly realizable by pair-state $R$-local OENN primitives using $D$ inter-site propagation layers.  The construction uses hidden pair-state variables $z_{q,p}\in V_p$ $(q,p\in P)$, so the hidden width at each carrier site scales as $\sum_{p\in P}\dim V_p$.

(iii) Universality. Under the same strong-connectivity hypothesis, for every continuous non-polynomial activation $\sigma:\R\to\R$, $\mathrm{OENN}^{R\text{-pair}}_\sigma(X,Y)$ is dense in $C_G(K,Y)$ in the uniform norm.  Moreover, the depth threshold is sharp in the worst-case sense that, for every $L<D$, the lower-bound statement applies to some pair of sites $p,q$ with $d_R(p,q)>L$.
\end{theorem}

\begin{corollary}[Pair-state cover-local OENN UAT]
\label{cor:cover-local-oenn-uat}
Let $R_{\mathrm{cov}}$ be the cover-local relation from Definition~\ref{def:R-local-oenn}.  If the undirected Hasse graph of $P$ is connected, then $\mathrm{OENN}^{R_{\mathrm{cov}}\text{-pair}}_\sigma(X,Y)$ is dense in $C_G(K,Y)$ in the uniform norm.
\end{corollary}

There are several remarks. The cover-local relation $R_{\mathrm{cov}}$ contains both upward and downward cover incidences.  This is necessary for universal approximation of arbitrary equivariant maps $X\to Y$, because an arbitrary output site may depend on features at any other site.  If communication is restricted to only one order direction, the directed communication graph of a nontrivial finite poset is generally not strongly connected.  In that case one can obtain universality only for maps satisfying the corresponding causal order-dependence restriction.

Theorem~\ref{thm:diameter-sharp-local-uat} is a universality theorem for the pair-state local completion.  Ordinary message passing aggregates neighbor information without a persistent source label; pair-state propagation stores $z_{q,p}$ and therefore lets the readout at $q$ access the contribution from each source $p$ separately.  Some mechanism of this kind (source labels, higher-order states, global attention, or a global broadcast) is unavoidable for worst-case approximation of arbitrary maps in $C_G(K,Y)$.

The hypothesis on $\sigma$ is the usual finite-dimensional MLP hypothesis.  ReLU and $\tanh$ satisfy it.  Coordinatewise activations need not commute with arbitrary nontrivial fiber transports; this is why OENN uses Reynolds blocks for stabilizer representations.

\section{Examples: Graphs and Sheaves}
\label{sec:examples}

\subsection{Fixed graphs via face-posets}
\label{subsec:graphs}

Let $\mathcal G=(V,E)$ be a finite undirected graph without parallel edges, and let $P:=V\sqcup E$ be its face-poset, with $v\leq e$ iff $v$ is an endpoint of $e$.  Every graph automorphism preserves incidence and hence acts on $P$ by poset automorphisms.

{\bf Vertex-only layers.}
Take permutation-only vertex fibers $V_v\simeq\R^{d_{\mathrm{in}}}$ and $W_v\simeq\R^{d_{\mathrm{out}}}$.  For a fixed graph, Proposition~\ref{prop:orbit-param} gives the most general linear $\Aut(\mathcal G)$-equivariant vertex operator:
\begin{equation}
\label{eq:graph-full-orbit-linear}
(Lh)_v=\sum_{\calO\in(V\times V)/\Aut(\mathcal G)} A_{\calO}
\sum_{\substack{u\in V\\(v,u)\in\calO}}h_u.
\end{equation}
For a general graph, the pair-orbits need not be only ``self'' and ``adjacent''.  If $\Aut(\mathcal G)$ is trivial, every ordered pair is its own orbit.  The standard message-passing linear layer
\begin{equation}
\label{eq:MP-linear}
(Lh)_v=A_{\mathrm{self}}h_v+A_{\mathrm{adj}}\sum_{u\in N(v)}h_u
\end{equation}
is nevertheless equivariant, because the diagonal relation $\Delta=\{(v,v)
:v\in V\}$ and the directed adjacency relation $E^{\pm}=\{(v,u):\{v,u\}\in E\}$ are unions of pair-orbits.  Thus \eqref{eq:MP-linear} is an OENN layer with additional parameter tying across all self-pair orbits and all adjacency-pair orbits.  It coincides with the full orbit-parametrized form only in the special case where the target-source pair-orbits in $V\times V$ are exactly the diagonal relation and the directed adjacency relation.  More generally, within the adjacency-supported masked subspace, it is full only when both $\Delta$ and $E^{\pm}$ are single pair-orbits; otherwise it additionally ties distinct orbit blocks and sets all non-adjacency orbit kernels to zero.

The usual nonlinear message-passing update
\begin{equation}
\label{eq:MP-nonlinear}
h'_v=\psi\left(h_v,\sum_{u\in N(v)}\phi(h_u)\right)
\end{equation}
is obtained from the invariant-relation aggregation above with $R=\Delta\cup E^{\pm}$: the explicit self input $h_v$ is the diagonal term, while one ties the encoders for all pair-orbits contained in $E^{\pm}$ and sums their summaries before the readout.  Hence ordinary MPNNs are OENN special cases, while the full OENN class can express finer orbit-dependent sharing on a fixed graph.

{\bf Local universality and MPNNs.}
For fixed graphs, the local UAT in Theorem~\ref{thm:diameter-sharp-local-uat} is concerned with the pair-state local completion.  In graph notation, the universal local hidden state is $z_{v,u}$: it is carried by vertex $v$ but remembers the source vertex $u$.  Local propagation updates
\[
z^{t+1}_{v,u}=\sum_{w\in N[v]} z^t_{w,u},
\]
where $N[v]$ includes the self-loop, followed by an equivariant rescaling after diameter-many steps.  This is a second-order/source-aware message-passing architecture on $V\times V$.  The ordinary MPNN \eqref{eq:MP-nonlinear} is recovered by discarding the persistent source label and tying messages by the adjacency relation; that tied anonymous subclass is not universal for arbitrary continuous equivariant maps on a fixed graph.

{\bf Cycles.}
Let $\mathcal G=C_n$ be the cycle with vertex set $\mathbb Z_n$.  Under the cyclic subgroup $C_n$, pair-orbits are directed offsets
\[
\calO_\ell=\{(i,i+\ell):i\in\mathbb Z_n\},
\qquad \ell\in\mathbb Z_n.
\]
Thus a $C_n$-equivariant linear layer is ordinary circular convolution:
\begin{equation}
\label{eq:cyclic-conv}
(Lh)_i=\sum_{\ell\in\mathbb Z_n}A_\ell h_{i+\ell}.
\end{equation}
Under the full dihedral group $D_{2n}$, reflection identifies offsets $\ell$ and $-\ell$.  The pair-orbits are indexed by undirected distance $k\in\{0,1,\ldots,\lfloor n/2\rfloor\}$, and the equivariant filters are
\begin{equation}
\label{eq:dihedral-filter}
\begin{aligned}
(Lh)_i
&=A_0h_i+\sum_{1\leq k<n/2}A_k(h_{i+k}+h_{i-k})\\
&\quad+\mathbf{1}_{2\mid n}A_{n/2}h_{i+n/2}.
\end{aligned}
\end{equation}
Therefore, for $n\geq 3$ and nonzero input and output fiber dimensions, the dihedral distance-isotropic filters are a strict subset of cyclic circular convolutions, obtained from \eqref{eq:cyclic-conv} by imposing $A_\ell=A_{-\ell}$ for every offset $\ell$.

{\bf Vertex-edge couplings.}
If edge fibers are included, then Proposition~\ref{prop:orbit-param} also describes vertex-edge and edge-vertex blocks.  Let $h_v\in\R^{d_V}$ be vertex features, $z_e\in\R^{d_E}$ edge features, and let the vertex output fiber be $\R^{d'_V}$.  A common incidence-tied vertex update is
\begin{equation}
\label{eq:vertex-edge-coupling}
(Lx)_v
=A_{VV,0}h_v
+A_{VV,1}\sum_{u\in N(v)}h_u
+A_{VE,\mathrm{inc}}\sum_{e\ni v}z_e,
\end{equation}
with $A_{VV,0},A_{VV,1}\in\Hom(\R^{d_V},\R^{d'_V})$ and $A_{VE,\mathrm{inc}}\in\Hom(\R^{d_E},\R^{d'_V})$.  Edge outputs can analogously use self-edge, incident-vertex, and edge-adjacency masks.  Relations such as $\{(v,e):v\leq e\}$ are unions of pair-orbits and therefore define equivariant orbital affine layers.  On an arbitrary fixed graph, however, \eqref{eq:vertex-edge-coupling} is generally a coarser tied special case of the full orbit-parametrized operator: in highly symmetric graphs the incidence mask may collapse to a small number of orbit types, while in asymmetric graphs it may split into many pair-orbits.

\subsection{Cellular and simplicial sheaf layers}
\label{subsec:sheaves}

Let $\mathcal K$ be a finite regular CW complex or simplicial complex with face-poset $P$, and let $G$ be a finite group acting on $P$ by order automorphisms.  A cellular sheaf assigns a vector space $V_p$ to each cell $p\in P$ and a restriction map
\[
R_{p\to q}:V_p\to V_q
\qquad(q\leq p),
\]
compatible along chains.  We use this contravariant face-poset convention throughout: $q\leq p$ means that $q$ is a face of $p$, and restrictions point from the larger cell to the face.  Reversing the poset recovers the opposite convention used in some neural-sheaf formulations.  Suppose this $G$-action lifts to the sheaf, meaning
\begin{equation}
\label{eq:sheaf-naturality}
T^V_{\gamma,q}R_{p\to q}=R_{\gamma p\to\gamma q}T^V_{\gamma,p}
\qquad(q\leq p).
\end{equation}

\begin{lemma}[Invariant fiber metrics and sheaf adjoints]
\label{lem:sheaf-adjoint-equivariance}
For every finite lifted action satisfying \eqref{eq:sheaf-naturality}, there exist inner products on all fibers $V_p$ for which every transport $T^V_{\gamma,p}:V_p\to V_{\gamma p}$ is an isometry.  With respect to any such invariant inner products, the adjoints of the restriction maps satisfy
\begin{equation}
\label{eq:sheaf-adjoint-naturality}
T^V_{\gamma,p}R^*_{p\to q}=R^*_{\gamma p\to\gamma q}T^V_{\gamma,q}
\qquad(q\leq p,
\ \gamma\in G).
\end{equation}
\end{lemma}

{\bf Linear sheaf-type orbital layers.}
Let $W=\bigoplus_{q\in P}W_q$ be a $G$-equivariant output bundle.  Consider the strict upward relation
\[
\mathcal U^{\circ}:=\{(q,p):q<p\}
\]
and the strict downward relation
\[
\mathcal D^{\circ}:=\{(q,r):r<q\}.
\]
Both are unions of pair-orbits, and diagonal self terms are treated separately.  For an upward pair-orbit $\calO\subset\mathcal U^{\circ}$ with representative $(q_{\calO},p_{\calO})$, choose
\[
B^\uparrow_{\calO}\in\Hom(V_{q_{\calO}},W_{q_{\calO}})
\]
and set
\[
A^\uparrow_{\calO}:=B^\uparrow_{\calO}R_{p_{\calO}\to q_{\calO}}
\in\Hom(V_{p_{\calO}},W_{q_{\calO}}).
\]
We require $A^\uparrow_{\calO}$ to satisfy the stabilizer condition \eqref{eq:stabilizer-intertwiner}.  Then Proposition~\ref{prop:orbit-param} transports $A^\uparrow_{\calO}$ to all pairs in $\calO$, giving an equivariant upward sheaf block.
A simple sufficient way to ensure this condition is to choose $B^\uparrow_{\calO}$ as an intertwiner for the target fiber representation,
\[
T^W_{\eta,q_{\calO}}B^\uparrow_{\calO}
=B^\uparrow_{\calO}T^V_{\eta,q_{\calO}}
\qquad(\eta\in H_{\calO}).
\]
Indeed, since $\eta$ fixes both $q_{\calO}$ and $p_{\calO}$, naturality gives
$T^V_{\eta,q_{\calO}}R_{p_{\calO}\to q_{\calO}}
=R_{p_{\calO}\to q_{\calO}}T^V_{\eta,p_{\calO}}$, and hence
$T^W_{\eta,q_{\calO}}A^\uparrow_{\calO}
=A^\uparrow_{\calO}T^V_{\eta,p_{\calO}}$.  This sufficient condition is not necessary; the most general requirement is still the stabilizer condition on the composed block $A^\uparrow_{\calO}$.

For downward terms, use the $G$-invariant inner products from Lemma~\ref{lem:sheaf-adjoint-equivariance}, so that adjoints are compatible with transports by \eqref{eq:sheaf-adjoint-naturality}.  If $r\leq q$, then
\[
R^*_{q\to r}:V_r\to V_q.
\]
For a downward pair-orbit $\calO\subset\mathcal D^{\circ}$ with representative $(q_{\calO},r_{\calO})$, choose
\[
\begin{aligned}
B^\downarrow_{\calO}
&\in\Hom(V_{q_{\calO}},W_{q_{\calO}}),\\
A^\downarrow_{\calO}
&:=B^\downarrow_{\calO}R^*_{q_{\calO}\to r_{\calO}}\\
&\in\Hom(V_{r_{\calO}},W_{q_{\calO}}),
\end{aligned}
\]
again satisfying the stabilizer condition.  Self terms are handled by blocks $A^0_{\calO}\in\Hom(V_{q_{\calO}},W_{q_{\calO}})$ on pair-orbits contained in the diagonal.
Analogously, the stronger but constructive condition
\[
T^W_{\eta,q_{\calO}}B^\downarrow_{\calO}
=B^\downarrow_{\calO}T^V_{\eta,q_{\calO}}
\qquad(\eta\in H_{\calO})
\]
implies the required stabilizer condition for $A^\downarrow_{\calO}$, because \eqref{eq:sheaf-adjoint-naturality} gives
$T^V_{\eta,q_{\calO}}R^*_{q_{\calO}\to r_{\calO}}
=R^*_{q_{\calO}\to r_{\calO}}T^V_{\eta,r_{\calO}}$.

The resulting linear layer has the form
\begin{equation}
\label{eq:sheaf-orbit-linear}
x'_q=
K^0(q,q)x_q+
\sum_{\substack{p:q<p}}K^\uparrow(q,p)x_p+
\sum_{\substack{r:r<q}}K^\downarrow(q,r)x_r,
\end{equation}
where the kernels are transported from the orbit representatives as in \eqref{eq:orbit-transport}.  The following lemma gives the explicit transporter-law verification.

\begin{lemma}[Sheaf orbital transporter law]
\label{lem:sheaf-orbit-transporter}
The self, upward, and downward kernels in \eqref{eq:sheaf-orbit-linear}, transported from the orbit representatives as in \eqref{eq:orbit-transport} and satisfying the stabilizer conditions above, satisfy the transporter law \eqref{eq:transporter-law}.  Hence \eqref{eq:sheaf-orbit-linear} is order-equivariant.
\end{lemma}

{\bf Special shared-fiber case.}  In the permutation-only case, or more generally when the fibers are globally identified and the shared blocks $B_0,B_\uparrow,B_\downarrow$ are chosen to intertwine all relevant transports and stabilizer actions, the orbit-wise notation reduces to the familiar shared formula
\begin{equation}
\label{eq:sheaf-simple-linear}
x'_q=B_0x_q+
\sum_{p:q<p}B_\uparrow R_{p\to q}x_p+
\sum_{r:r<q}B_\downarrow R^*_{q\to r}x_r,
\end{equation}
where $B_\uparrow$ and $B_\downarrow$ act after the restriction or adjoint has landed in the fiber at $q$.  This is a special case of the orbit-wise transported layer in \eqref{eq:sheaf-orbit-linear}.  

{\bf Nonlinear sheaf aggregation.}
A nonlinear sheaf layer is obtained by applying the pair-orbit aggregation construction to the strict comparable-pair orbits in $\mathcal U^{\circ}$ and $\mathcal D^{\circ}$, together with a pointwise self term on the diagonal.  On an upward orbit, form the pair feature $R_{p\to q}x_p\in V_q$ in the auxiliary pair bundle with fiber $V_q$ over $(q,p)$; naturality \eqref{eq:sheaf-naturality} gives the required transporter law.  On a downward orbit, form $R^*_{q\to r}x_r\in V_q$; the adjoint naturality \eqref{eq:sheaf-adjoint-naturality} gives the corresponding transporter law.  The encoders are chosen invariant under the corresponding pair stabilizers, and the final cellwise readout is a pointwise Reynolds layer.  This gives an equivariant nonlinear sheaf-type OENN layer and correctly handles orientation signs or other nontrivial stabilizer actions without double-counting diagonal contributions.

{\bf UAT in the sheaf-indexed setting.}
By Theorem~\ref{thm:OENN-UAT}, the full OENN class on the face-poset of a fixed complex is dense in the continuous maps between total sheaf-valued feature spaces that are equivariant for the chosen lifted finite group action, uniformly on compact invariant sets.  
If the undirected Hasse graph of the face-poset is connected, Corollary~\ref{cor:cover-local-oenn-uat} compiles the broadcast used in the proof into finitely many cover-local up/down propagation layers on pair-states $(q,p)$, where $q$ is the carrier cell and $p$ is the source cell.  Consequently, pair-state cover-local sheaf-indexed OENNs are also universal, provided depth is allowed to scale at least with the Hasse-graph diameter and source-cell-labeled memory is allowed.  
Note that, without such source-aware global mixing, local sheaf diffusion or sheaf message-passing layers such as \eqref{eq:sheaf-simple-linear} remain equivariant but are not universal in general.

\section{Conclusion}
\label{sec:conclusion}

We have developed order-equivariant neural networks (OENNs) as a unified framework for learning on structured domains whose coordinates are organized not only by a set or a graph, but by an ordered incidence geometry.  We have shown that equivariant deep learning can be formulated in terms of equivariant bundles over posets: symmetries move sites, transports move local feature coordinates, and neural layers are precisely constrained by the resulting compatibility laws.  This viewpoint simultaneously recovers familiar graph message passing, higher-order vertex-edge interactions, and cellular or simplicial sheaf layers, while also identifying the strictly more expressive orbit-wise and pair-state constructions that are invisible in ordinary anonymous aggregation.

Beyond the unification of architectures, we have shown that full OENNs give universal approximation for continuous order-equivariant maps on compact invariant domains, while the local theory separates genuine locality from mere parameter tying: bounded-depth local message passing is not universally expressive in general, but source-aware pair-state propagation restores universality once information can traverse the Hasse graph.  Thus the framework supplies a principled methodology for choosing architectures: use coarse relation-message passing when the task only requires tied local aggregation, and use orbit-wise, sheaf-aware, or pair-state mechanisms when the task depends on incidence type, orientation, stabilizer actions, or long-range structured interactions.

The broader applicability of OENN is not limited to the examples treated explicitly here.  Many scientific and geometric data sets are naturally indexed by cells, regions, events, sensors, strata, or multiscale components equipped with incidence or restriction relations.  In such settings, the present theory suggests a systematic route from domain structure to architecture: specify the ordered or categorical index space, lift the relevant symmetries to feature fibers, and derive equivariant layers from transporter and stabilizer constraints rather than by ad hoc sharing rules.  

The appendix connects OENN with CENN (category-equivariant neural network)~\cite{Maruyama25CENN,Maruyama25HARRep,Maruyama25HARArch,Maruyama26CatReg,Maruyama26InfinityHAR,MaruyamaYasuda26Grothendieck}. This strand of research demonstrates how the extended equivariant models help to improve the experimental performance of classic equivariant models in concrete machine learning tasks, showing that geometric deep learning can be extended from groups of symmetries to categories of transformations (including non-invertible multi-object transformations).

\section*{Impact Statement}

This paper aims to advance symmetry-aware machine learning research.  There seem to be no societal consequences of this theoretical research on geometric and topological deep learning that would need to be specifically highlighted here.

\appendix

\section{Category-Equivariant Neural Networks and Their Relation with OENNs}
\label{app:cenn-oenn}

This appendix explains the framework of category-equivariant neural networks (CENNs) and clarifies how the OENN framework relates to it; experimental applications in the related literature are reviewed as well~\cite{Maruyama25CENN,Maruyama25HARRep,Maruyama25HARArch,Maruyama26CatReg,Maruyama26InfinityHAR,MaruyamaYasuda26Grothendieck}. CENN is concerned with categorical equivariance with respect to transformations beyond invertible symmetries (i.e., non-invertible symmetries on/between multiple objects). Non-invertible categorical symmetry is extensively studied in the recent development of mathematical and theoretical physics. CENN explores the same idea of non-invertible categorical symmetry in the context of geometric deep learning. 

In a group-equivariant network, one fixes representations $\rho_X,\rho_Y$ of a group $G$ and requires $F(\rho_X(g)x)=\rho_Y(g)F(x)$. Equivalently, the network commutes with all invertible changes of viewpoint encoded by the one-object category obtained from $G$.  CENN replaces this one-object group by a possibly multi-object category $\mathcal C$.  Its objects can represent typed feature sites, cells, regions, sensor states, local contexts, or levels of resolution, while its arrows can represent admissible transports, restrictions, incidence relations, causal updates, coarse-to-fine maps, or local symmetries.  Thus the arrows of $\mathcal C$ describe not only which symmetries should be respected, but also which transformations are allowed to move information between different feature types. The naturality condition gives the categorical generalization of equivariance. 

Formally, input and output feature spaces are organized as contravariant functors to the category $\mathbf{Meas}$ of measurable spaces and measurable maps  
\(
X,Y:\mathcal C^{\mathrm{op}}\longrightarrow \mathbf{Meas}
\)
or as functors to topological vector spaces equipped with their Borel structures.  
A category-equivariant layer is then a family of maps $F_a:X(a)\to Y(a)$ satisfying the naturality equation
\(
Y(u)\circ F_a=F_b\circ X(u)
\)
where $u:b\to a$. 
This is the categorical extension of equivariance.  Note that $u$ need not be invertible and the two ends of $u$ need not have the same type.  Consequently, the same equation can express ordinary group equivariance, compatibility with order restrictions, invariance under irreversible time updates, sheaf-like local-to-global consistency, and orbitwise consistency for local symmetries.

This viewpoint gives a unified language for the non-invertible and multi-object transformations listed in Table~\ref{tab:cat-equivariance}~\cite{Maruyama25CENN}.  Since arrows in a category compose, categorical equivariance is not merely a list of separate parameter-sharing constraints.  The constraints must be compatible with categorical composition.  For example, if two restrictions or transports can be composed in the data domain, then the corresponding neural maps must compose coherently as well.  This functorial compatibility is what distinguishes categorical equivariance from an ad hoc collection of tied weights, masks, or pooling rules.
\begin{table}[!h]
\centering
\caption{Categorical equivariance across domains \cite{Maruyama25CENN}.}
\renewcommand{\arraystretch}{1.15}
\begin{tabular}{@{}lll@{}}
\hline
\textbf{Category} & \textbf{Arrows encode} & \textbf{Equivariance enforces}\\
\hline
Group & invertible sym. & global sym. consistency\\
Monoid & causal update & time consistency\\
Poset &  hierarchical rel. & hierarchical consistency\\
Lattice & logical relation & logical consistency\\
Graph & adjacency rel. & message-passing consist.\\
Sheaf & spatial relation & local-to-global consist.\\
Groupoid & local symmetry & orbitwise consistency\\
Gen.\ cat. & typed process & compositional consist.\\
\hline
\end{tabular}
\label{tab:cat-equivariance}
\end{table}

Beyond single types of categorical symmetry such as groups, graphs, sheaves and logics (lattices), CENN can accommodate \emph{compositional symmetries}, such as Group$\times$Poset (product category), and \emph{contextual symmetries}, such as $\int_{c\in\mathcal C}F(c)$ (Grothendieck construction), which integrates local symmetries $F(c)$ over the base context category $\mathcal C$.  For instance, a signal may carry both a global group symmetry and a hierarchical incidence structure represented by a poset; categorically, this can be modeled by the product category of the group and the poset (where both are seen as categories). Moreover, local symmetry types may vary over a base context category $\mathcal C$.  Let $F(c)$ denote the local category of transformations available at context $c$. Then, the Grothendieck construction $$\int_{c\in\mathcal C}F(c)$$ assembles these varying local symmetries into one category.  This can be regarded as the categorical extension of semidirect product of groups: instead of one group acting uniformly everywhere, the available transformations can depend on object type, context, or stratum, while still composing coherently.

In the following subsections of the appendix, we give a more formal account of CENN and analyze the relation between CENN and OENN. We show, in particular, that completed OENN input-output maps are embedded into CENN through the action groupoid of the chosen symmetry action, while hidden pair-state and branch constructions may use auxiliary action groupoids built from the same action.  Note that they are not embedded through the thin poset category alone.  The poset order supplies incidence, masks, covers, and sheaf-type restriction structure; the equivariance in OENN is supplied by the group acting by order automorphisms.

\subsection{CENN basics}
\label{app:cenn-basics}

We first explain the basics of the CENN formalism.  A category serves as an indexing device for the transformations under which a network must be equivariant.  Objects are feature sites or feature types, and arrows are admissible changes of viewpoint, restrictions, transports, or symmetries.  

{\bf Categorical index spaces.}
Formally, a small category $\mathcal C$ consists of a set of objects $\operatorname{Ob}(\mathcal C)$, arrow sets $\Hom_{\mathcal C}(b,a)$, identity arrows $\id_a:a\to a$, and an associative composition law.  We write
\[
s(u)=b,
\qquad
 t(u)=a
\]
for the source and target of an arrow $u:b\to a$.  If $v:c\to b$ and $u:b\to a$, then $u\circ v:c\to a$.

Three basic examples are useful.
\begin{enumerate}[leftmargin=1.5em,itemsep=2pt,topsep=2pt]
\item A group $G$ is a one-object category $BG$ with $\Hom(\ast,\ast)=G$; naturality over $BG$ is ordinary group equivariance.
\item A poset $(P,\leq)$ is a thin category with one arrow $p\to q$ iff $p\leq q$; naturality over it is compatibility with order-restriction maps.
\item If $G$ acts on a set $S$, the action groupoid $G\ltimes S$ has objects $s\in S$ and arrows $s\to\gamma s$ labeled by $\gamma\in G$; naturality over it is equivariance under the action of $G$.
\end{enumerate}

To introduce categorical convolutional layers, we also need topology and measures on arrows.  A measured topological category is a small category such that each hom-set $\Hom_{\mathcal C}(b,a)$ is a second-countable locally compact Hausdorff space, composition is continuous, and each hom-set carries a $\sigma$-finite Radon measure $\mu_{b,a}$.  We additionally impose the local countability/Radon condition needed by the category-convolution formula we use below: for each target object $a$, the set of sources $b$ with nonzero incoming measure is countable, and the incoming-arrow space
\begin{equation}
\label{eq:app-cenn-incoming-space}
I(a):=\bigsqcup_{b\in\operatorname{Ob}(\mathcal C)}\Hom_{\mathcal C}(b,a),
\qquad
\mu_a:=\bigoplus_b\mu_{b,a}
\end{equation}
with the disjoint-union topology is itself second-countable locally compact Hausdorff and carries the $\sigma$-finite Radon measure $\mu_a$.  Equivalently, one may take the measurable topological spaces $(I(a),\mu_a)$ as part of the CENN data, with the displayed hom-set restrictions.  Thus $I(a)$ is the space of all arrows through which the $a$-component of a layer can read information.  In finite specializations, including the OENN specialization below, the topology is discrete, $\mu_{b,a}$ is counting measure, and integrals over $I(a)$ are finite sums.

{\bf Feature functors.}
Let $\mathbf{Meas}$ denote the category of measurable spaces and measurable maps.  
A CENN feature space is a contravariant functor
\[
Z:\mathcal C^{\mathrm{op}}\longrightarrow \mathbf{Meas}.
\]
Thus each object $a$ carries a measurable space $Z(a)$, and each arrow $u:b\to a$ induces a measurable pullback/transport
\[
Z(u):Z(a)\longrightarrow Z(b),
\]
satisfying
\[
Z(\id_a)=\id_{Z(a)},
\qquad
Z(u\circ v)=Z(v)\circ Z(u)
\]
for $v:c\to b$ and $u:b\to a$.   In the analytic layers below, each $Z(a)$ is the Borel measurable space underlying a topological vector space and each structural transport $Z(u)$ is continuous linear.  The contravariant convention means that an arrow into $a$ lets a field on $a$ be restricted or transported to the arrow source.

The analytic CENN model realizes these measurable spaces as Borel spaces underlying topological vector spaces of continuous local fields.  For every object $a$, choose a compact base space $\Omega(a)$ and a finite-dimensional real fiber $E_Z(a)$, and set
\begin{equation}
\label{eq:app-cenn-field-space}
Z(a):=C(\Omega(a),E_Z(a)).
\end{equation}
We equip this field space with the compact-open topology, equivalently the sup-norm topology in the present compact finite-dimensional setting, and with its Borel $\sigma$-algebra.
For each arrow $u:b\to a$, choose continuous base maps
\[
\pi_u:\Omega(b)\to\Omega(a),
\qquad
\tau_u:\Omega(a)\to\Omega(b),
\]
and a linear fiber transport
\[
L^Z_u:E_Z(a)\to E_Z(b).
\]
Here $\pi_u$ defines the pullback of fields, while $\tau_u$ is the sampling map used by the convolutional kernel.  Functoriality requires, for $v:c\to b$ and $u:b\to a$,
\[
\pi_{u\circ v}=\pi_u\circ\pi_v,
\qquad
\tau_{u\circ v}=\tau_v\circ\tau_u,
\qquad
L^Z_{u\circ v}=L^Z_v\circ L^Z_u,
\]
with identity maps on identity arrows.  The induced action on local fields is
\begin{equation}
\label{eq:app-cenn-functor-action}
Z(u)h:=L^Z_u\circ h\circ\pi_u
\in C(\Omega(b),E_Z(b)),
\quad h\in Z(a).
\end{equation}
When every base is a point, $\Omega(a)=\{\ast\}$, this reduces to the finite-dimensional point-base case
\[
Z(a)\cong E_Z(a),
\qquad
Z(u)=L^Z_u:E_Z(a)\to E_Z(b).
\]
This is the case used to embed OENN below, with finite-dimensional vector spaces understood as Borel measurable spaces.

{\bf Equivariance as naturality.}
Let $X,Y:\mathcal C^{\mathrm{op}}\to\mathbf{Meas}$ be feature functors whose object spaces also carry the topologies used for approximation.

\begin{definition}[Category-equivariant maps]
\label{def:app-cenn-nat}
A \emph{continuous category-equivariant map} from $X$ to $Y$ is a natural transformation in $\mathbf{Meas}$
\[
\Phi:X\Rightarrow Y,
\qquad
\Phi=\{\Phi_a:X(a)\to Y(a)\}_{a\in\operatorname{Ob}(\mathcal C)},
\]
whose components are continuous maps.  Equivalently, every $\Phi_a$ is measurable and continuous, and for every arrow $u:b\to a$, the square
\[
\begin{array}{ccc}
X(a) & \xrightarrow{\Phi_a} & Y(a)\\
\scriptstyle X(u)\downarrow && \downarrow\scriptstyle Y(u)\\
X(b) & \xrightarrow{\Phi_b} & Y(b)
\end{array}
\]
commutes.  Written as an equation,
\begin{equation}
\label{eq:app-cenn-naturality}
Y(u)\circ\Phi_a=\Phi_b\circ X(u).
\end{equation}
No linearity of the components $\Phi_a$ is assumed.  We write $\operatorname{EqvCont}(X,Y)$ for the space of such natural transformations in $\mathbf{Meas}$ with continuous components.
\end{definition}

The approximation topology on $\operatorname{EqvCont}(X,Y)$ is the compact-open finite-object topology.  In the normed local-field setting used for the UAT, a basic seminorm is specified by a finite set $F\subset\operatorname{Ob}(\mathcal C)$ and compact sets $K_a\subset X(a)$:
\begin{equation}
\label{eq:app-cenn-compact-open-seminorm}
\|\Phi\|_{F,(K_a)}:=
\max_{a\in F}\sup_{x\in K_a}\|\Phi_a(x)\|_\infty .
\end{equation}
Approximation therefore means uniform approximation on finitely many objects and compact subsets of their feature spaces.

This definition contains the usual equivariance notions.  If $\mathcal C=BG$ has one object and arrows a group $G$, then a contravariant functor is a representation up to the inverse convention, and \eqref{eq:app-cenn-naturality} is the usual equation 
$$F(\rho_X(g)x)=\rho_Y(g)F(x).$$  
If $\mathcal C$ is a thin poset category, naturality is compatibility with order arrows.  If $\mathcal C$ is an incidence or face category, naturality is sheaf-like compatibility with incidence restrictions.

{\bf Categorical convolution.}
The basic affine CENN layer is a convolution over incoming arrows.  Let $Z,Z':\mathcal C^{\mathrm{op}}\to\mathbf{Meas}$ be field-type feature functors represented by local vector fields as above, with continuous linear structural transports.  A category kernel assigns, for every incoming arrow $u:b\to a$ and point $y\in\Omega(a)$, a linear map
\[
\mathsf K(u,y):E_Z(b)\to E_{Z'}(a).
\]
Assume the standard regularity needed for the integral below: measurability in $u$, continuity in $y$, and an $L^1$ bound over $I(a)$ on compact parameter ranges.  A bias is a family $\beta_a\in Z'(a)$ satisfying
\[
Z'(u)\beta_a=\beta_b
\qquad(u:b\to a).
\]
For $z\in Z(a)$, define
\begin{equation}
\label{eq:app-cenn-convolution}
\begin{aligned}
\bigl((\widetilde L_{\mathsf K})_az\bigr)(y)
&:=\beta_a(y)\\
&\quad+\int_{I(a)}\mathsf K(u,y)
\bigl(Z(u)z\bigr)(\tau_u y)\,d\mu_a(u).
\end{aligned}
\end{equation}
Here $u:b\to a$, so $Z(u)z\in Z(b)$ and $(Z(u)z)(\tau_u y)\in E_Z(b)$, exactly the domain of $\mathsf K(u,y)$.

The kernel is called natural when the affine family is a continuous natural transformation $Z\Rightarrow Z'$ in $\mathbf{Meas}$.  Equivalently, for every arrow $w:a\to c$,
\begin{equation}
\label{eq:app-cenn-conv-natural}
Z'(w)(\widetilde L_{\mathsf K})_c
=
(\widetilde L_{\mathsf K})_aZ(w).
\end{equation}
Written pointwise for $z\in Z(c)$ and $y\in\Omega(a)$, the non-bias terms satisfy
\[
\begin{aligned}
&L^{Z'}_w\!\int_{I(c)}\!\mathsf K(u',\pi_wy)
\bigl(Z(u')z\bigr)(\tau_{u'}\pi_wy)\,d\mu_c(u')
\\
&\qquad=
\int_{I(a)}\!\mathsf K(u,y)
\bigl(Z(w\circ u)z\bigr)(\tau_uy)\,d\mu_a(u).
\end{aligned}
\]
For finite point-base categories, \eqref{eq:app-cenn-convolution} becomes the finite sum
\begin{equation}
\label{eq:app-cenn-finite-conv}
(\widetilde L_{\mathsf K})_az
=
\beta_a+\sum_{u:b\to a}\mathsf K(u)Z(u)z.
\end{equation}
Thus CENN convolution is the categorical analog of a tied-weight equivariant affine layer.

{\bf Natural nonlinearities.}
Nonlinearities must also commute with the category action.  Let $S:\mathcal C^{\mathrm{op}}\to\mathbf{Meas}$ be the scalar feature functor
\[
S(a)=C(\Omega(a),\R),
\qquad
S(u)r=r\circ\pi_u,
\]
equipped with the compact-open Borel measurable structure.  A scalar channel is a natural transformation $s:Z\Rightarrow S$ in $\mathbf{Meas}$ whose components are continuous.  Given a scalar activation $\alpha:\R\to\R$, the associated scalar-gated nonlinearity is
\begin{equation}
\label{eq:app-cenn-gate}
\bigl(\Sigma^{\alpha,s}_a z\bigr)(y)
=
\alpha\bigl(s_a(z)(y)\bigr)z(y).
\end{equation}
Naturality of $s$ implies naturality of $\Sigma^{\alpha,s}$.  With the trivial scalar functor $S$ above, the usual coordinatewise activation is recovered from scalar gates only when the relevant coordinate projection is itself a natural scalar channel.  In the finite point-base case this means that the coordinate is fixed by every stabilizer element acting on its fiber.  For a finite-dimensional trivial scalar coordinate, adding a constant channel gives
\[
r\longmapsto (1,r_j)
\longmapsto \alpha(r_j)(1,r_j),
\]
whose first component is $\alpha(r_j)$.

On a nontrivial permutation representation, an individual coordinate projection is generally not natural as a map to the trivial scalar functor: if a stabilizer element sends coordinate $j$ to coordinate $k$, naturality would force $r_j=r_k$ for all vectors $r$.  Thus scalar gates into $S$ do not by themselves implement unrestricted coordinatewise activations on permutation branches.  Coordinatewise activations on equivariant collections of scalar slots must instead be typed as natural maps for the corresponding permutation-scalar functor, or else be placed inside the branch MLP and equivariantized by the Reynolds construction used for OENN.  For arbitrary nontrivial stabilizer representations, raw coordinate projections are generally not natural, and coordinatewise activations need not commute with transport; this is why the OENN construction uses natural scalar channels or Reynolds-equivariant pointwise blocks rather than unrestricted coordinatewise nonlinearities.

\begin{definition}[CENN]
\label{def:app-cenn}
Fix a continuous non-polynomial scalar activation $\alpha$; in the general CENN UAT, it is also assumed globally Lipschitz.  A \emph{category-equivariant neural network} from $X$ to $Y$ is a finite composition of continuous natural layers in $\mathbf{Meas}$ generated by category-convolutions, scalar-gated nonlinearities, finite direct sums/parallel channels, and arrow-bundle lift/convolution/compilation layers~\cite{Maruyama25CENN}.  
The resulting class is denoted
\[
\mathsf{CENN}_\alpha(X,Y).
\]
\end{definition}

The detailed definitions of these arrow-bundle generators are in~\cite{Maruyama25CENN}.  In the finite action-groupoid specialization used below, only their structural consequences are needed: finite direct sums of transported arrow samples, finite CENN convolutions on those arrow-indexed bundles, and finite equivariant averages that compile the arrow-bundle features back to the target functor.

Because natural transformations in $\mathbf{Meas}$ with continuous components are closed under composition and finite vector-space direct sums equipped with their Borel $\sigma$-algebras,
\[
\mathsf{CENN}_\alpha(X,Y)\subseteq \operatorname{EqvCont}(X,Y).
\]

{\bf Equivariant universal approximation.}
The CENN UAT of~\cite{Maruyama25CENN} says that, under mild analytic hypotheses on the category, the converse inclusion holds after closure.  The hypotheses have the following operational meaning.
\begin{enumerate}[leftmargin=1.5em,itemsep=2pt,topsep=2pt]
\item \emph{Approximate identities:} kernels on $I(a)$ can localize near a chosen incoming arrow.  In finite categories these are Dirac masses $\delta_u$.
\item \emph{Arrow-probe separation:} the transported samples appearing in \eqref{eq:app-cenn-convolution} separate points on the compact sets being approximated.  In finite point-base categories, identity arrows recover the original features, so separation is automatic.
\item \emph{Equivariant compilation:} approximants constructed on arrow-bundle features can be reassembled into a natural transformation with values in the target functor.  In finite groupoids, this is finite averaging over the relevant arrows or stabilizers.
\end{enumerate}

\begin{theorem}[CENN UAT]
\label{thm:app-cenn-uat}
Assume the standard CENN hypotheses above, with $X$ and $Y$ field-type $\mathbf{Meas}$-valued functors whose object spaces carry the topologies and norms used in the compact-open seminorms.  Let $\alpha:\R\to\R$ be continuous, non-polynomial, and globally Lipschitz.  Then
\[
\overline{\mathsf{CENN}_\alpha(X,Y)}
=
\operatorname{EqvCont}(X,Y)
\]
in the compact-open finite-object topology.  Equivalently, for every $\Phi\in\operatorname{EqvCont}(X,Y)$, every finite $F\subset\operatorname{Ob}(\mathcal C)$, every compact $K_a\subset X(a)$ for $a\in F$, and every $\varepsilon>0$, there exists $\Psi\in\mathsf{CENN}_\alpha(X,Y)$ such that
\[
\max_{a\in F}\sup_{x\in K_a}
\|\Phi_a(x)-\Psi_a(x)\|_\infty<\varepsilon.
\]
\end{theorem}

For the reader mainly interested in OENN, the finite specialization is the essential case: $\mathcal C$ is a finite discrete action groupoid, all bases are points, all integrals are sums, approximate identities are point masses, arrow probes separate because identity arrows are present, and equivariant compilation is finite groupoid averaging.  The next subsection makes this specialization explicit.

\subsection{Relationship with OENNs}
\label{app:oenn-as-cenn}

We now identify completed OENN maps over $P$ as finite action-groupoid, point-base specializations of CENN; auxiliary OENN states are represented over the action groupoids of their corresponding finite $G$-sets.  Let $(P,\leq)$ be a finite poset, let $G$ be a finite group acting on $P$ by order automorphisms, and let
\[
X_{\mathrm{tot}}=\bigoplus_{p\in P}V_p,
\qquad
Y_{\mathrm{tot}}=\bigoplus_{q\in P}W_q
\]
be OENN input and output bundles with transports $T^V_{\gamma,p}$ and $T^W_{\gamma,q}$ as in Definition~\ref{def:bundle}.

Define the action groupoid
\[
\mathcal A:=G\ltimes P.
\]
Its objects are elements of $P$, and an arrow
\[
(\gamma,p):p\longrightarrow \gamma p
\]
is specified by $\gamma\in G$ and $p\in P$.  Composition is
\[
(\eta,\gamma p)\circ(\gamma,p)=(\eta\gamma,p).
\]
This is a finite discrete category, equipped with counting measure.  The ordinary OENN bundle $V$ is equivalently the point-base CENN feature functor
\[
\begin{aligned}
\mathcal V&:\mathcal A^{\mathrm{op}}\to\mathbf{Meas},\\
\mathcal V(p)&=V_p,\\
\mathcal V(\gamma,p)&=(T^V_{\gamma,p})^{-1}:V_{\gamma p}\to V_p.
\end{aligned}
\]
Here each finite-dimensional vector space carries its Borel measurable structure, and the transport maps are measurable linear isomorphisms.  Contravariant functoriality is exactly the cocycle identity for the transports.  The same construction gives a functor $\mathcal W$ from the output bundle.  This objectwise bundle functor is useful, but it is not yet the correct encoding of an arbitrary OENN map $X_{\mathrm{tot}}\to Y_{\mathrm{tot}}$, because an OENN output at a site $q$ may depend on all input sites, not only on $V_q$.

For maps between total feature spaces, define instead the \emph{global-input} functor
\[
\begin{aligned}
\mathfrak X&:\mathcal A^{\mathrm{op}}\to\mathbf{Meas},\\
\mathfrak X(q)&:=X_{\mathrm{tot}},\\
\mathfrak X(\gamma,q)&:=\rho_X(\gamma^{-1}):
X_{\mathrm{tot}}\to X_{\mathrm{tot}},
\end{aligned}
\]
and the site-output functor
\[
\begin{aligned}
\mathfrak Y&:\mathcal A^{\mathrm{op}}\to\mathbf{Meas},\\
\mathfrak Y(q)&:=W_q,\\
\mathfrak Y(\gamma,q)&:=(T^W_{\gamma,q})^{-1}:W_{\gamma q}\to W_q.
\end{aligned}
\]
Here $(\gamma,q):q\to\gamma q$ is an arrow of $\mathcal A$.  Point bases are understood, so $\mathfrak X(q)$ and $\mathfrak Y(q)$ are finite-dimensional vector spaces equipped with their Borel measurable structures, rather than spaces of nontrivial fields; their structural maps are continuous linear maps.

These global-input/site-output functors classify a completed map from a full source state to sitewise outputs.  They are not meant to type an arbitrary OENN stack as a single composable chain of CENN natural transformations: after such a completed map, the codomain object at $q$ is $W_q$, whereas the next OENN layer would again read a full hidden total state.  For the layerwise formulation, let $H=\bigoplus_{s\in S}H_s$ be an intermediate equivariant bundle over a finite $G$-set $S$, with total action $\rho_H$, and let $H'=\bigoplus_{t\in T}H'_t$ be the next target bundle over a finite $G$-set $T$.  On the target action groupoid $G\ltimes T$, use the total-source functor
\[
\underline H(t):=H_{\mathrm{tot}},
\qquad
\underline H(\gamma,t):=\rho_H(\gamma^{-1}):H_{\mathrm{tot}}\to H_{\mathrm{tot}},
\]
together with the site-output functor
\[
\mathfrak H'(t):=H'_t,
\qquad
\mathfrak H'(\gamma,t):=(T^{H'}_{\gamma,t})^{-1}:H'_{\gamma t}\to H'_t.
\]
Thus an OENN layer $H_{\mathrm{tot}}\to H'_{\mathrm{tot}}$ is represented by a natural family $\underline H\Rightarrow\mathfrak H'$, and layers are composed after assembling these families into maps between total feature spaces.

For a compact $G$-invariant set $K\subset X_{\mathrm{tot}}$, let
\[
\operatorname{Nat}_K(\mathfrak X,\mathfrak Y)
\]
denote families of continuous maps $\Phi_q:K\to W_q$ satisfying the naturality equation on $K$.

\begin{proposition}[OENN equivariance as CENN naturality]
\label{prop:app-oenn-nat}
The assignment
\[
F\longmapsto \Phi^F,
\qquad
\Phi^F_q(x):=F(x)_q,
\]
defines a canonical linear bijection
\[
C_G(K,Y_{\mathrm{tot}})
\cong
\operatorname{Nat}_K(\mathfrak X,\mathfrak Y).
\]
Under this bijection, the uniform norm on $C_G(K,Y_{\mathrm{tot}})$ induced by the product max norm on $Y_{\mathrm{tot}}$ is exactly the compact-open finite-object seminorm with $F=P$ and $K_q=K$ for all $q\in P$.
\end{proposition}

\begin{proof}
Naturality for the arrow $(\gamma,q):q\to\gamma q$ says
\[
\mathfrak Y(\gamma,q)\Phi_{\gamma q}(x)
=
\Phi_q\bigl(\mathfrak X(\gamma,q)x\bigr),
\qquad x\in K.
\]
By the definitions of $\mathfrak X$ and $\mathfrak Y$, this is
\begin{equation}
\label{eq:app-nat-as-equiv}
(T^W_{\gamma,q})^{-1}\Phi_{\gamma q}(x)
=
\Phi_q(\rho_X(\gamma^{-1})x).
\end{equation}
Since $K$ is $G$-invariant, replacing $x$ by $\rho_X(\gamma)x$ gives the equivalent equation
\begin{equation}
\label{eq:app-coordinate-equivariance}
\Phi_{\gamma q}(\rho_X(\gamma)x)
=
T^W_{\gamma,q}\Phi_q(x).
\end{equation}
If $F(x)_q:=\Phi_q(x)$, then \eqref{eq:app-coordinate-equivariance} is precisely the $\gamma q$-coordinate of
\[
F(\rho_X(\gamma)x)=\rho_Y(\gamma)F(x).
\]
Thus natural families are exactly continuous order-equivariant maps.  Finally,
\[
\begin{aligned}
&\max_{q\in P}\sup_{x\in K}\|\Phi_q(x)-\Psi_q(x)\|\\
&\quad=\sup_{x\in K}
\|F_\Phi(x)-F_\Psi(x)\|_{Y_{\mathrm{tot}},\max},
\end{aligned}
\]
which proves the topological assertion.
\end{proof}

\begin{proposition}[OENN primitives as finite CENN primitives]
\label{prop:app-oenn-layers-cenn}
With the primitivewise typing convention above, each completed OENN primitive, viewed through its site-output family after assembly, is represented by finite point-base CENN primitives over action groupoids of the finite $G$-sets indexing that primitive and its auxiliary branch states.  Thus the primitive-level comparison uses action groupoids $G\ltimes S$ for auxiliary finite $G$-sets $S$ that occur in the OENN stack, rather than forcing every hidden primitive to live literally over the single object set $P$.
\begin{enumerate}[leftmargin=1.5em]
\item An OENN orbital affine layer is an identity-supported category-convolution from the relevant total-source functor to the site-output target functor.
\item A pointwise Reynolds layer over any finite $G$-set $S$ is a finite CENN stack over $G\ltimes S$, together with the finite branch auxiliary action groupoids used to realize its Reynolds block.
\item Pair-state OENN layers are the same construction over the auxiliary action groupoid $G\ltimes(P\times P)$, with additional $G$-stable support restrictions encoding locality.
\end{enumerate}
\end{proposition}

\begin{proof}
Let $S$ and $T$ be finite $G$-sets, and let
\[
H_{\mathrm{tot}}=\bigoplus_{s\in S}H_s,
\qquad
H'_{\mathrm{tot}}=\bigoplus_{t\in T}H'_t
\]
be equivariant bundles.  An orbital affine layer between them has the form
\[
A(h)_t=b_t+\sum_{s\in S}K(t,s)h_s.
\]
Write
\[
L_t:H_{\mathrm{tot}}\to H'_t,
\qquad
L_t(h):=\sum_{s\in S}K(t,s)h_s.
\]
The naturality equation for the corresponding completed affine family
$\underline H\Rightarrow\mathfrak H'$ over $G\ltimes T$ is
\[
(T^{H'}_{\gamma,t})^{-1}L_{\gamma t}
=
L_t\rho_H(\gamma^{-1}),
\qquad
(T^{H'}_{\gamma,t})^{-1}b_{\gamma t}=b_t.
\]
Equivalently,
\[
L_{\gamma t}\rho_H(\gamma)=T^{H'}_{\gamma,t}L_t,
\qquad
b_{\gamma t}=T^{H'}_{\gamma,t}b_t.
\]
Expanding the first identity in the $H_s$-source blocks gives exactly the transporter law
\[
K(\gamma t,\gamma s)T^H_{\gamma,s}
=
T^{H'}_{\gamma,t}K(t,s),
\]
and the second identity is the OENN fixed-bias law.  With counting measure and point bases, define a category kernel supported only at the identity arrow of each target object by
\[
\mathsf K(\id_t):=L_t,
\qquad
\mathsf K(u):=0\quad(u\neq\id_t\text{ in }I(t)).
\]
Then the CENN convolution formula gives $(\widetilde L_{\mathsf K}h)_t=b_t+L_th$, and its naturality condition is exactly the equation above.  Hence orbital affine layers, including those between auxiliary finite $G$-sets, are identity-supported CENN category-convolutions.

For a pointwise Reynolds layer over a finite $G$-set $S$, the stabilizer average
\[
\psi_{s_0}(u)=\frac1{|G_{s_0}|}\sum_{h\in G_{s_0}}
(T^F_{h,s_0})^{-1}\Psi_{s_0}(T^E_{h,s_0}u)
\]
is a finite groupoid average.  Its branch realization is typed with finite auxiliary branch action groupoids whose objects are the branch labels $h\in G_{s_0}$ appearing in Lemma~\ref{lem:reynolds-realization}, transported along the orbit of $s_0$; the stabilizer permutes these labels, while each branch object carries ordinary finite-dimensional trivial scalar coordinates.  The affine maps in the branch realization are affine natural maps between these branch functors, and the scalar activation is applied by scalar gates only to the trivial scalar coordinates of each branch object, as in \eqref{eq:app-cenn-gate}.  If the same branch labels are folded into a single fiber $(\R^n)^{G_{s_0}}$, the coordinatewise activation remains equivariant for the permutation action, but it should be regarded as the corresponding natural permutation-feature nonlinearity, or as part of the Reynolds pointwise primitive, not as a scalar gate into the trivial scalar functor $S$.  Transporting the representative block along the orbit by
\[
\psi_{\gamma s_0}(e)=T^F_{\gamma,s_0}\psi_{s_0}\bigl((T^E_{\gamma,s_0})^{-1}e\bigr)
\]
is precisely the naturality condition in $G\ltimes S$.  Thus pointwise Reynolds layers are finite CENN stacks.

Finally, pair-state layers replace $S$ by $P\times P$ with diagonal action $\gamma(q,p)=(\gamma q,\gamma p)$.  A locality mask such as
\[
K((q,p),(r,s))=0
\quad\text{unless }s=p\text{ and }(q,r)\in R
\]
is not a new equivariance law; it is a $G$-stable support restriction on an otherwise ordinary finite action-groupoid CENN kernel.
\end{proof}

\begin{lemma}[Finite CENN UAT layers are OENN-realizable]
\label{lem:app-finite-cenn-oenn-realization}
Let $\mathcal A=G\ltimes P$ be the finite action groupoid above, with point bases and finite-dimensional Borel feature spaces.  The finite CENN approximants needed in the CENN UAT proof over $\mathcal A$ can be chosen to be realizable by full OENN primitives after allowing the auxiliary finite $G$-sets already permitted in Definition~\ref{def:oenn}.  Consequently, if $\Psi$ is one of these chosen finite action-groupoid CENN approximants from Theorem~\ref{thm:app-cenn-uat}, then the assembled map
\[
F_\Psi:X_{\mathrm{tot}}\to Y_{\mathrm{tot}},
\qquad
F_\Psi(x)_q:=\Psi_q(x),
\]
belongs to $\mathrm{OENN}^{\mathrm{full}}_\alpha(X_{\mathrm{tot}},Y_{\mathrm{tot}})$.
\end{lemma}

\begin{proof}
A point-base CENN hidden functor $Z:\mathcal A^{\mathrm{op}}\to\mathbf{Meas}$ in the finite construction has finite-dimensional vector spaces $Z(q)$ with invertible linear transports.  For an arrow $(\gamma,q):q\to\gamma q$, write
\[
T^Z_{\gamma,q}:=Z(\gamma,q)^{-1}:Z(q)\to Z(\gamma q).
\]
Functoriality of $Z$ is exactly the cocycle identity for the transports $T^Z$, so the family $\{Z(q),T^Z_{\gamma,q}\}$ is an equivariant bundle over the finite $G$-set $P$.  The same observation applies to any finite direct sum or arrow-bundle hidden functor: its slots are indexed by a finite $G$-set built functorially from arrows of $\mathcal A$ or finite tuples of such arrows, and the structural maps are the corresponding bundle transports.

Consider first an affine CENN layer in the layerwise total-source form
$\Lambda:\underline H\Rightarrow\mathfrak H'$ over $G\ltimes T$, where
$H_{\mathrm{tot}}=\bigoplus_{s\in S}H_s$.  Write
\[
\Lambda_t(h)=c_t+\sum_{s\in S}\Lambda_{t,s}h_s,
\qquad
\Lambda_{t,s}:H_s\to H'_t .
\]
Naturality for the arrow $(\gamma,t):t\to\gamma t$ says
\[
(T^{H'}_{\gamma,t})^{-1}\Lambda_{\gamma t}(h)
=
\Lambda_t(\rho_H(\gamma^{-1})h).
\]
Equivalently, after replacing $h$ by $\rho_H(\gamma)h$ and comparing the $H_s$-blocks,
\[
\Lambda_{\gamma t,\gamma s}T^H_{\gamma,s}
=
T^{H'}_{\gamma,t}\Lambda_{t,s},
\qquad
c_{\gamma t}=T^{H'}_{\gamma,t}c_t .
\]
Thus a global-input affine CENN layer is exactly an orbital affine OENN layer between the associated equivariant bundles over the finite $G$-sets $S$ and $T$.

If an affine hidden layer is presented objectwise as $Z\Rightarrow Z'$ over a single action groupoid $G\ltimes S$, write $\Lambda_s(z)=c_s+\lambda_sz$.  Naturality gives
\[
\lambda_{\gamma s}T^Z_{\gamma,s}=T^{Z'}_{\gamma,s}\lambda_s,
\qquad
c_{\gamma s}=T^{Z'}_{\gamma,s}c_s,
\]
which is the same transporter law with diagonal source support.  In a multilayer stack, each assembled codomain total space is then retyped as the next total-source functor, as described above.  Finite category-convolutions are affine natural layers of this form, since the integral in \eqref{eq:app-cenn-convolution} is a finite sum in the relevant finite action groupoid.  The arrow-bundle lift, probe, and compilation maps used in the CENN UAT proof~\cite{Maruyama25CENN} are built from finite direct sums of transports, coordinate inclusions/projections, and finite groupoid averages; hence they are affine natural maps and therefore orbital affine OENN layers by the preceding argument.

It remains only to check nonlinearities.  We use the finite UAT construction in branch-MLP form: scalar gates are used only for coordinates that are trivial natural scalar channels, as described after \eqref{eq:app-cenn-gate}.  Branch spaces carrying a nontrivial permutation action are not treated as if their individual coordinate projections were natural for the trivial scalar functor.  When the same scalar activation is applied coordinatewise to such an equivariant array, it is used as a vector-valued equivariant branch operation, not as a scalar gate into $S$; in the OENN realization used here, it is placed inside the ordinary branch MLP and then equivariantized by the Reynolds construction of Lemma~\ref{lem:reynolds-realization}.  For nontrivial stabilizer actions, the resulting equivariant pointwise operation is therefore represented by a pointwise Reynolds layer.  We do not assert that an arbitrary scalar gate in an arbitrary CENN architecture is exactly an OENN primitive; the claim is only that the finite approximants required for the action-groupoid UAT can be chosen inside the full OENN primitive class.  Since full OENNs are closed under finite composition and finite parallel concatenation, each chosen finite CENN approximant in the finite action-groupoid case assembles to a full OENN.
\end{proof}

\begin{corollary}[CENN UAT and OENN UAT]
\label{cor:app-cenn-oenn-uat}
For the action groupoid $\mathcal A=G\ltimes P$ and the global-input/site-output functors $\mathfrak X,\mathfrak Y$ above, write
\[
\mathsf{CENN}_\alpha(\mathfrak X,\mathfrak Y)|_K
\]
for the restrictions to $K$ of CENN natural transformations.  Then the compact-domain form of the CENN UAT gives
\[
\overline{\mathsf{CENN}_\alpha(\mathfrak X,\mathfrak Y)|_K}
=
\operatorname{Nat}_K(\mathfrak X,\mathfrak Y)
\cong
C_G(K,Y_{\mathrm{tot}})
\]
for every compact $G$-invariant $K\subset X_{\mathrm{tot}}$, whenever the activation satisfies the general CENN UAT hypotheses.  By Lemma~\ref{lem:app-finite-cenn-oenn-realization}, the finite action-groupoid CENN approximants may be chosen to assemble to full OENNs, so the same density statement gives the categorical form of the full OENN UAT
\[
\overline{\mathrm{OENN}^{\mathrm{full}}_\alpha(X_{\mathrm{tot}},Y_{\mathrm{tot}})|_K}
=
C_G(K,Y_{\mathrm{tot}}).
\]
For globally Lipschitz continuous non-polynomial $\alpha$, this follows from the general CENN theorem and the finite-realization lemma above.  The OENN theorem in Section~\ref{sec:uat} is slightly sharper in activation assumptions: because the finite action-groupoid case reduces to finite-dimensional MLP approximation plus finite Reynolds averaging, it only assumes that $\alpha$ is continuous and non-polynomial.
\end{corollary}

\begin{proof}
The groupoid $G\ltimes P$ is finite and discrete.  Hence the CENN approximate identities are point masses, all convolution integrals are finite sums, arrow probes separate because identity arrows are present, and equivariant compilation is finite groupoid averaging.  Therefore Theorem~\ref{thm:app-cenn-uat} applies to $\mathfrak X$ and $\mathfrak Y$.

The theorem is stated for global natural transformations, while OENN is stated on a compact invariant domain $K$.  This causes no loss in the present finite-dimensional groupoid case.  Given $F\in C_G(K,Y_{\mathrm{tot}})$, extend it coordinatewise by the Tietze extension theorem to a continuous map $H:X_{\mathrm{tot}}\to Y_{\mathrm{tot}}$, and then Reynolds-symmetrize
\[
\widetilde F(x):=\frac1{|G|}\sum_{\gamma\in G}\rho_Y(\gamma^{-1})H(\rho_X(\gamma)x).
\]
Then $\widetilde F$ is continuous and $G$-equivariant, and $\widetilde F|_K=F$ because $K$ is $G$-invariant and $F$ is equivariant.  Thus every compact-domain natural family extends to a global natural transformation of $\mathfrak X$ into $\mathfrak Y$.

Applying the CENN UAT to this extension with $F=P$ and $K_q=K$ for all $q$ gives density on $K$ by finite action-groupoid CENN approximants.  Proposition~\ref{prop:app-oenn-nat} identifies the target space of natural families with $C_G(K,Y_{\mathrm{tot}})$ and identifies the CENN compact-open finite-object seminorm with the OENN uniform norm on $K$.  Lemma~\ref{lem:app-finite-cenn-oenn-realization} shows that the required finite approximants can be chosen to assemble to full OENN networks, while Proposition~\ref{prop:app-oenn-layers-cenn} gives the primitive-level embedding of OENN layers into the finite CENN formalism; together they complete the comparison without requiring arbitrary CENN layers to be OENN primitives.
\end{proof}

\begin{remark}[The poset category does not give the OENN embedding]
\label{rem:app-thin-poset-not-oenn}
The thin category associated with $P$ has a unique arrow $p\to q$ when $p\le q$.  Naturality over this category imposes commutation with order arrows, such as restriction or incidence maps.  OENN equivariance instead imposes
\[
F(\rho_X(\gamma)x)=\rho_Y(\gamma)F(x)
\qquad(\gamma\in G),
\]
where $G$ acts by order automorphisms.  Thus the order relation determines useful $G$-stable masks, cover relations, pair-state communication graphs, and sheaf restriction structure, but the categorical symmetry behind a completed OENN input-output map is the action groupoid $G\ltimes P$.  Hidden OENN primitives may additionally use auxiliary finite $G$-sets, such as $P\times P$, and those primitives are typed over their own action groupoids.  In short, OENN embeds into finite point-base CENN over the action groupoids $G\ltimes S$ for the finite $G$-sets $S$ appearing in the construction.  Completed input-output maps are represented over $G\ltimes P$; pair-state layers are represented over $G\ltimes(P\times P)$; Reynolds branch MLPs use finite branch-copy action groupoids; and locality is imposed by optional $G$-stable support constraints inherited from the poset.
\end{remark}

\subsection{Experimental Applications to Activity Recognition with Grothendieck Equivariant Networks}

In CENN, a category serves as a generalized symmetry structure whose arrows may encode both invertible changes of viewpoint and non-invertible transformations between data configurations. We evaluate this principle on CatHAR, a controlled multi-sensor human-activity-recognition benchmark in which subject-dependent sensor frames give local $\mathrm{SO}(3)$ gauge symmetries and sensor merging gives non-invertible configuration morphisms~\cite{MaruyamaYasuda26Grothendieck}.

{\bf Grothendieck CENN: integrating symmetries.}
The \emph{Grothendieck construction} is a fundamental composition method in the theory of fibrations: it packages a family of categories that varies over a base context category into a single total category. Formally, given a functor valued in the category of small categories
\begin{equation}
F:\mathcal{B}^{\mathrm{op}}\to \mathbf{Cat},
\end{equation}
the Grothendieck construction produces the integrated category
\begin{equation}
\int_{b\in \mathcal{B}} F(b),
\end{equation}
which is obtained by gluing the local categorical structures $F(b)$ along the base context $\mathcal{B}$. A \emph{Grothendieck CENN} is a CENN that is equivariant with respect to this integrated category, enabling machine learning under a \emph{global composite symmetry} assembled from \emph{context-dependent local symmetries}.

{\bf CatHAR experiment and results.}
CatHAR contains $1080$ length-$96$ tri-axial sequences from $18$ subjects and six activity classes. Five IMUs are placed at the torso, left and right arms, and left and right legs; each subject has an independently rotated local frame at each sensor. The subject-wise split uses $12$ subjects for training, $3$ for validation, and $3$ for testing, and every model is trained on the fine five-sensor configuration. Table~\ref{tab:cathar-fine} reports mean fine-configuration test accuracy and macro-F1 over five independent runs~\cite{MaruyamaYasuda26Grothendieck}.

\begin{table}[!h]
\centering
\small
\setlength{\tabcolsep}{4pt}
\caption{Mean fine-configuration test accuracy and macro-F1 on CatHAR over five independent runs.~\cite{MaruyamaYasuda26Grothendieck}}
\label{tab:cathar-fine}
\begin{tabular}{lcc}
\hline
Model & Acc. & F1 \\
\hline
Convolutional LSTM & $0.493$ & $0.467$ \\
Patch Transformer & $0.633$ & $0.582$ \\
Grothendieck Equivariant Net & $\mathbf{0.978}$ & $\mathbf{0.978}$ \\
\hline
\end{tabular}
\end{table}

{\bf Formal definition of Grothendieck category.}
Let $\mathcal{B}$ be a small category, called the \emph{base context} category, and let
$$
F:\mathcal{B}^{\mathrm{op}}\to\mathbf{Cat}
$$
be a contravariant functor.
The \emph{Grothendieck construction} $\int_{b\in\mathcal{B}} F(b)$ is the category defined as follows:
\begin{itemize}
\item Objects are pairs $(b,x)$ with $b\in\mathrm{Ob}(\mathcal{B})$ and $x\in\mathrm{Ob}(F(b))$.
\item An arrow $(b,x)\to(b',y)$ is a pair $(f,\phi)$ where $f:b\to b'$ in $\mathcal{B}$ and $\phi:x\to F(f)(y)$ in $F(b)$.
\item Composition is given by
\[
(g,\psi)\circ(f,\phi)
=
(g\circ f,\;F(f)(\psi)\circ \phi).
\]
\end{itemize}
Thus $\int_{b\in \mathcal{B}} F(b)$ glues the fiber categories $F(b)$ along the context arrows of $\mathcal{B}$ into one total category.

{\bf Categorical structure of CatHAR.}
Let $P$ be the five physical IMUs. An object $b$ of the base category $\mathcal{B}$ is a sensor configuration $\pi_b:P\twoheadrightarrow V_b$, and a morphism $u:b\to b'$ satisfies $\pi_{b'}=u\circ\pi_b$; it therefore coarsens the configuration by merging logical sensor nodes and is generally non-invertible. The fine configuration has
$V_{b_f}=\{T,LA,RA,LL,RL\}$, while the coarse configuration has $V_{b_c}=\{T,A,L\}$. The morphism $q:b_f\to b_c$ merges $LA,RA$ into $A$ and $LL,RL$ into $L$.
Over each configuration $b$ lies the one-object gauge groupoid $\mathbf{B}G_b$, where
\[
G_b=\mathrm{SO}(3)^{V_b}
\]
represents independent changes of the local sensor frames. For a coarsening $u:b\to b'$, the pullback $(u^*h)_v=h_{u(v)}$ defines the contravariant groupoid-valued functor
$$
F:\mathcal{B}^{\mathrm{op}}\to\mathbf{Cat}.
$$
The total category
$$
\int_{b\in\mathcal{B}}F(b)
$$
consequently has morphisms that combine local gauge transformations with configuration coarsenings.

{\bf Architecture and naturality.}
Grothendieck Equivariant Net (GroEN) first transports all sensor vectors to a common torso reference frame using the available relative rotations. It then applies a multiplicity-aware pushforward to a canonical configuration, constructs rotation-invariant channels from norms, temporal-difference norms, dot products, and cosine similarities, and passes the resulting sequence to a patch-based Transformer encoder and an MLP classifier. Transport reduces the independent gauge changes to one common rotation, the scalar contractions remove that residual rotation, and functoriality of the pushforward ensures compatibility with composition of coarsenings.
If $\mathcal{X}$ is the data functor and $f_b$ is the predictor at configuration $b$, this compatibility is the naturality condition
\[
f_{b'}\bigl(\mathcal{X}(u,k)(\xi)\bigr)=f_b(\xi)
\]
for every morphism $(u,k):b\to b'$ in the Grothendieck symmetry category. This is an architectural guarantee of invariance to the modeled gauge changes and consistency under the modeled non-invertible coarsenings.

{\bf Interpretation of the fine-configuration results.}
CatHAR is constructed so that three activity classes have identical per-sensor magnitude and frequency statistics and differ only through the relative $\mathrm{SO}(3)$ geometry of the arm and leg signals. Held-out-subject evaluation therefore rewards gauge-consistent cross-sensor relational features rather than subject-specific coordinate-frame cues. Among the three displayed models, GroEN obtains the highest mean accuracy and macro-F1, both $0.978$. The margin over Transformer is $0.345$ in accuracy and $0.396$ in macro-F1, indicating a substantial advantage on this geometry-sensitive task. This performance is consistent with its direct construction of invariant relational features before temporal classification. Convolutional LSTM must infer cross-sensor alignment from raw channels, whereas Patch Transformer processes the same untransported measurements with patch-based attention; neither explicitly incorporates the available relative rotations. GroEN performs this alignment before temporal encoding, so the classifier receives relational features that remain stable across the subject-specific sensor frames used for held-out evaluation. The experimental protocol and the separate coarse-configuration consistency evaluation are given in~\cite{MaruyamaYasuda26Grothendieck}.

\section{Proofs}
\label{app:main-proofs}

This appendix provides the proofs of the results stated in Sections~\ref{sec:oenn}-\ref{sec:examples}.  The numbering below refers to the corresponding statements in the main text.

\subsection{Proofs for Section~\ref{sec:oenn}}

\begin{proof}[Proof of Proposition~\ref{prop:transporter-law}]
Using \eqref{eq:rhoX-def} and \eqref{eq:block-kernel},
\[
\begin{aligned}
(L\rho_X(\gamma)x)_q
&=\sum_{p\in P}K(q,p)T^V_{\gamma,\gamma^{-1}p}
  x_{\gamma^{-1}p}\\
&=\sum_{r\in P}K(q,\gamma r)T^V_{\gamma,r}x_r.
\end{aligned}
\]
On the other hand,
\[
\begin{aligned}
(\rho_Y(\gamma)Lx)_q
&=T^W_{\gamma,\gamma^{-1}q}(Lx)_{\gamma^{-1}q}\\
&=\sum_{r\in P}T^W_{\gamma,\gamma^{-1}q}
  K(\gamma^{-1}q,r)x_r.
\end{aligned}
\]
Equality for all $x$ is equivalent to
\[
K(q,\gamma r)T^V_{\gamma,r}=T^W_{\gamma,\gamma^{-1}q}K(\gamma^{-1}q,r)
\qquad(q,r\in P).
\]
Replacing $q$ by $\gamma q$ gives \eqref{eq:transporter-law}.  The converse is the same computation in reverse.
\end{proof}

\begin{proof}[Proof of Proposition~\ref{prop:orbit-param}]
If $\gamma_1(q_{\calO},p_{\calO})=\gamma_2(q_{\calO},p_{\calO})$, then $\eta:=\gamma_2^{-1}\gamma_1\in H_{\calO}$.  By functoriality and \eqref{eq:stabilizer-intertwiner},
\[
\begin{aligned}
&T^W_{\gamma_1,q_{\calO}}A_{\calO}(T^V_{\gamma_1,p_{\calO}})^{-1} \\
&\quad =T^W_{\gamma_2,q_{\calO}}T^W_{\eta,q_{\calO}}
  A_{\calO}(T^V_{\eta,p_{\calO}})^{-1}
  (T^V_{\gamma_2,p_{\calO}})^{-1} \\
&\quad =T^W_{\gamma_2,q_{\calO}}A_{\calO}
  (T^V_{\gamma_2,p_{\calO}})^{-1}.
\end{aligned}
\]
Thus \eqref{eq:orbit-transport} is well-defined.  Substitution into \eqref{eq:transporter-law} gives equivariance.  Conversely, if $L$ is equivariant, set $A_{\calO}:=K(q_{\calO},p_{\calO})$.  Taking $\gamma\in H_{\calO}$ in \eqref{eq:transporter-law} gives \eqref{eq:stabilizer-intertwiner}, and arbitrary $\gamma$ gives \eqref{eq:orbit-transport}.  The dimension formula follows by the direct sum over pair-orbits.
\end{proof}

\begin{proof}[Proof of Lemma~\ref{lem:reynolds-realization}]
Equivariance follows by the change of variables $t=hg$:
\[
\begin{aligned}
\mathcal R_H^{\mathrm{eq}}[\Psi](T^U_gu)
&=\frac1{|H|}\sum_{h\in H}(T^V_h)^{-1}\Psi(T^U_{hg}u)\\
&=T^V_g\mathcal R_H^{\mathrm{eq}}[\Psi](u).
\end{aligned}
\]
For the realization, write the MLP as
\[
\Psi(u)=A_L\sigma(A_{L-1}\sigma(\cdots\sigma(A_1u+b_1)\cdots)+b_{L-1})+b_L.
\]
For a hidden width $n_i$, use the branch space $E_i=(\R^{n_i})^H$ with right-regular action
\[
(\pi_i(g)z)_h=z_{hg}.
\]
The first affine map is
\[
(B_1u)_h=A_1T^U_hu+b_1,
\]
and the hidden affine maps are $(B_i z)_h=A_i z_h+b_i$.  These maps are $H$-equivariant, and coordinatewise $\sigma$ is equivariant because $H$ only permutes branches.  The output map
\[
Cz=\frac1{|H|}\sum_{h\in H}(T^V_h)^{-1}(A_Lz_h+b_L)
\]
is affine and $H$-equivariant.  The branch indexed by $h$ computes $\Psi(T^U_hu)$, so the resulting network computes \eqref{eq:eq-reynolds-block}.
\end{proof}

\begin{proof}[Proof of Lemma~\ref{lem:pair-summary-equiv}]
If $(q,p)=\gamma(q_{\calO},p_{\calO})$, then $(\tau q,\tau p)=\tau\gamma(q_{\calO},p_{\calO})$.  Hence
\[
\begin{aligned}
\phi_{\tau q,\tau p}\bigl(T^V_{\tau,p}v\bigr)
&=\phi_{\calO}\bigl((T^V_{\tau\gamma,p_{\calO}})^{-1}
  T^V_{\tau,p}v\bigr)\\
&=\phi_{\calO}\bigl((T^V_{\gamma,p_{\calO}})^{-1}v\bigr)\\
&=\phi_{q,p}(v).
\end{aligned}
\]
The map $p\mapsto\tau p$ bijects the terms in the sum defining $S_{\calO}(q;x)$ with those defining $S_{\calO}(\tau q;\rho_X(\tau)x)$.
\end{proof}

\begin{proof}[Proof of Proposition~\ref{prop:pair-aggregation-oenn}]
Well-definedness in the choice of $\gamma$ follows from $G_{q_0}$-equivariance of $\psi_{q_0}$ and Lemma~\ref{lem:pair-summary-equiv}.  Equivariance follows by replacing $q=\gamma q_0$ with $\tau q=(\tau\gamma)q_0$ and using the same lemma.

For the realization, let $S=P\times P$ with the diagonal $G$-action and define an auxiliary bundle $\widehat Z$ over $S$ by
\[
\widehat Z_{(q,p)}:=V_p,
\qquad
T^{\widehat Z}_{\tau,(q,p)}:=T^V_{\tau,p}:\widehat Z_{(q,p)}\to\widehat Z_{(\tau q,\tau p)}.
\]
Define the pair-lift
\[
B:X\to\bigoplus_{(q,p)\in P\times P}\widehat Z_{(q,p)},
\qquad
(Bx)_{(q,p)}:=x_p.
\]
It is an orbital linear map: for every $\tau\in G$,
\[
(B\rho_X(\tau)x)_{(q,p)}
=T^V_{\tau,\tau^{-1}p}x_{\tau^{-1}p}
=(\rho_{\widehat Z}(\tau)Bx)_{(q,p)}.
\]
For each pair-orbit $\calO$, restrict $\widehat Z$ to $\calO$ and apply the pointwise Reynolds layer whose representative block is the invariant block $\phi_{\calO}:V_{p_{\calO}}\to\R^{m_{\calO}}$, with trivial transports on the output fibers.  By \eqref{eq:pointwise-transport}, this produces pair features
\[
u^{\calO}_{(q,p)}=\phi_{q,p}(x_p)\qquad ((q,p)\in\calO).
\]
Next define the summary bundle $M^{\calO}$ over $P$ by $M^{\calO}_q=\R^{m_{\calO}}$ with trivial transports, and use the orbital linear summation map
\[
(A^{\calO}u)_q:=\sum_{\substack{p\in P\\(q,p)\in\calO}}u_{(q,p)}.
\]
Its support is the $G$-stable relation $\{(q,(q,p)):(q,p)\in\calO\}\subset P\times\calO$, and its nonzero kernels are identities, so it satisfies the auxiliary transporter law.  Hence $A^{\calO}u$ is exactly the summary $S_{\calO}(q;x)$.

Finally, concatenate the identity skip $x_q\in V_q$ with all summaries belonging to
\[
\mathcal A(q):=\{\calO:\exists p\in P\text{ with }(q,p)\in\calO\}.
\]
This gives an equivariant bundle $E$ over $P$ with
\[
E_q:=V_q\oplus\bigoplus_{\calO\in\mathcal A(q)}\R^{m_{\calO}},
\]
where transports act as $T^V$ on $V_q$ and trivially on the summary coordinates.  The final map with representative blocks $\psi_{q_0}$ is a pointwise Reynolds layer $E\to W$, and it computes \eqref{eq:pair-aggregation-layer}.  Therefore $F$ is a finite composition of OENN primitives.
\end{proof}

\begin{proof}[Proof of Proposition~\ref{prop:mp-compatible}]
The first claim is Proposition~\ref{prop:pair-aggregation-oenn}.  For the approximation claim, apply the finite-group density lemma, Lemma~\ref{lem:finite-group-density}, to each stabilizer-invariant encoder and each stabilizer-equivariant readout on the finitely many orbit representatives.  Since the number of orbits and sites is finite, the resulting errors compose continuously and can be chosen small enough to give uniform approximation on the prescribed compact set.
\end{proof}

\begin{proof}[Proof of Proposition~\ref{prop:oenn-architecture-hierarchy}]
The first inclusion is realized locally as follows.  Given a relation-message-passing layer of the form \eqref{eq:relation-mp-oenn}, first write the input into a pair-state diagonal seed
\[
z^0_{q,p}:=
\begin{cases}
h_q, & p=q,\\
0, & p\neq q,
\end{cases}
\]
which is a carrier-diagonal write map from the $P$-indexed state to the $P\times P$-indexed pair-state bundle.  Then apply one source-preserving pair-state $R$-local propagation
\[
z^1_{q,p}:=\sum_{r:(q,r)\in R}z^0_{r,p}.
\]
Since $z^0_{r,p}$ is nonzero only when $r=p$, this gives $z^1_{q,p}=h_p$ for $(q,p)\in R$ and $z^1_{q,p}=0$ otherwise.  Thus the apparent off-diagonal lift is implemented by an allowed diagonal seed followed by one pair-state local propagation.  Apply the transported invariant encoders slotwise on the relevant pair-orbits $\calO\subseteq R$, sum the resulting features over source slots at fixed carrier $q$ by a carrier-local affine map, and apply the carrier-local Reynolds readout $\theta_q$.  Thus every relation-message-passing layer belongs to $\mathrm{OENN}^{R\text{-pair}}_\sigma$.  The second inclusion follows because pair-state bundles are auxiliary finite $G$-set bundles allowed in the full OENN class, and $R$-local masks are special orbital masks.
\end{proof}

\subsection{Proofs for Section~\ref{sec:uat}}

\begin{proof}[Proof of Lemma~\ref{lem:sitewise-stabilizer}]
This is the $q$-coordinate of $F(\rho_X(\eta)x)=\rho_Y(\eta)F(x)$; since $\eta q=q$, the $q$-coordinate of $\rho_Y(\eta)y$ is $T^W_{\eta,q}y_q$.
\end{proof}

\begin{proof}[Proof of Lemma~\ref{lem:finite-group-density}]
Fix the norm $\|\cdot\|_*$ appearing in the statement.  Choose an $H$-invariant inner product on $V$ by averaging any inner product over $H$, and let $\|\cdot\|_0$ be its norm.  Then every $T^V_h$ is an isometry for $\|\cdot\|_0$.  Since $V$ is finite-dimensional, there is a constant $c>0$ such that
\[
\|v\|_*\leq c\|v\|_0\qquad(v\in V).
\]
The standard finite-dimensional MLP universal approximation theorem applies to scalar continuous functions on compact boxes.  For vector-valued $f$, apply the scalar theorem coordinatewise in a basis of $V$; for an arbitrary compact set $K\subset U$, first extend each coordinate continuously to a compact box containing $K$ by the Tietze extension theorem.  Thus, for the continuous non-polynomial activation $\sigma$~\cite{Leshno93}, choose $\Psi$ with
\[
\sup_{u\in K}\|\Psi(u)-f(u)\|_0<\varepsilon/c.
\]
The Reynolds average $\psi$ is $H$-equivariant by Lemma~\ref{lem:reynolds-realization}.  For $u\in K$,
\[
\psi(u)-f(u)=\frac1{|H|}\sum_{h\in H}(T^V_h)^{-1}\bigl(\Psi(T^U_hu)-f(T^U_hu)\bigr),
\]
because $f(T^U_hu)=T^V_hf(u)$.  Since $K$ is $H$-invariant and the $T^V_h$ are isometries for $\|\cdot\|_0$, the $\|\cdot\|_0$-norm of the right-hand side is at most $\varepsilon/c$.  Therefore
\[
\sup_{u\in K}\|\psi(u)-f(u)\|_*<\varepsilon.
\]
This proves the density statement for the original norm.  The branch realization is Lemma~\ref{lem:reynolds-realization}.
\end{proof}

\begin{proof}[Proof of Theorem~\ref{thm:OENN-UAT}]
Fix $f\in C_G(K,Y)$ and $\varepsilon>0$.  Choose norms on each $W_q$ and use the max norm $\|y\|_Y=\max_q\|y_q\|_{W_q}$.  Let $\calQ\subset P$ be a set of site-orbit representatives.  For $q_0\in\calQ$, define
\[
g_{q_0}:K\to W_{q_0},\qquad g_{q_0}(x)=f(x)_{q_0}.
\]
By Lemma~\ref{lem:sitewise-stabilizer}, $g_{q_0}$ is $G_{q_0}$-equivariant.

Let
\[
C:=\max\left(1,\max_{\gamma\in G,\ q_0\in\calQ}\|T^W_{\gamma,q_0}\|_{\mathrm{op}}\right),
\]
computed with the chosen norms.  Set $\delta=\varepsilon/C$.  Applying Lemma~\ref{lem:finite-group-density} with $H=G_{q_0}$, $U=X$, $V=W_{q_0}$, the norm $\|\cdot\|_{W_{q_0}}$, and $K$ gives a Reynolds block
\[
\psi_{q_0}:X\to W_{q_0}
\]
such that $\psi_{q_0}$ is $G_{q_0}$-equivariant and
\begin{equation}
\label{eq:representative-approx}
\sup_{x\in K}\|\psi_{q_0}(x)-g_{q_0}(x)\|_{W_{q_0}}<\delta.
\end{equation}
For $q=\gamma q_0$, define
\begin{equation}
\label{eq:uat-transported-psi}
\psi_q(x):=T^W_{\gamma,q_0}\psi_{q_0}(\rho_X(\gamma^{-1})x).
\end{equation}
This is independent of the chosen $\gamma$.  Indeed, if $\gamma q_0=\gamma' q_0$, then $\eta=\gamma'^{-1}\gamma\in G_{q_0}$ and $\gamma=\gamma'\eta$.  For $v=\rho_X(\gamma'^{-1})x$, $G_{q_0}$-equivariance gives
\[
T^W_{\eta,q_0}\psi_{q_0}(\rho_X(\eta^{-1})v)=\psi_{q_0}(v),
\]
which implies equality of the two definitions.

Define $F:X\to Y$ by $F(x)_q=\psi_q(x)$.  If $q=\gamma q_0$ and $\tau\in G$, then $\tau q=(\tau\gamma)q_0$, so
\[
F(\rho_X(\tau)x)_{\tau q}
=T^W_{\tau\gamma,q_0}\psi_{q_0}(\rho_X(\gamma^{-1})x)
=T^W_{\tau,q}F(x)_q.
\]
Thus $F$ is order-equivariant.

For $x\in K$ and $q=\gamma q_0$, equivariance of $f$ gives
\[
f(x)_q=T^W_{\gamma,q_0}g_{q_0}(\rho_X(\gamma^{-1})x).
\]
Since $K$ is $G$-invariant, $\rho_X(\gamma^{-1})x\in K$, and \eqref{eq:representative-approx} gives
\[
\|F(x)_q-f(x)_q\|_{W_q}
\leq \|T^W_{\gamma,q_0}\|_{\mathrm{op}}\delta
\leq C\delta=\varepsilon.
\]
Taking the maximum over $q$ gives $\|F(x)-f(x)\|_Y\leq\varepsilon$ on $K$.

It remains to verify that $F$ is a full OENN\@.  Let $\widetilde X$ be the bundle over $P$ with fiber $\widetilde X_q=X$ and transport $T^{\widetilde X}_{\gamma,q}=\rho_X(\gamma)$.  The broadcast map
\begin{equation}
\label{eq:global-broadcast}
B:X\to\bigoplus_{q\in P}\widetilde X_q,
\qquad (Bx)_q=x,
\end{equation}
is orbital linear; indeed $B\rho_X(\gamma)=\rho_{\widetilde X}(\gamma)B$ follows immediately from $T^{\widetilde X}_{\gamma,q}=\rho_X(\gamma)$.  The pointwise map $N:\widetilde X\to Y$ defined by $(Nz)_q=\psi_q(z_q)$ is a pointwise Reynolds layer by Definition~\ref{def:pointwise-reynolds-layer}.  Hence $F=N\circ B$ belongs to $\mathrm{OENN}^{\mathrm{full}}_\sigma(X,Y)$.  This proves density.
\end{proof}

\begin{proof}[Proof of Corollary~\ref{cor:perm-only-uat}]
This is the specialization of Theorem~\ref{thm:OENN-UAT} to trivial original fiber transports.  Corollary~\ref{cor:perm-only} identifies the linear layers with pair-orbit tying, and Lemma~\ref{lem:reynolds-realization} realizes the finite stabilizer averages needed for equivariant readouts.
\end{proof}

\begin{proof}[Proof of Theorem~\ref{thm:diameter-sharp-local-uat}]
The dependency claim in the lower bound is an induction on the number of inter-site layers.  It is true at depth zero because each site contains only its own input.  A pointwise map, a carrier-diagonal equivariant map, or a concatenation does not enlarge the dependency set.  An $R$-local inter-site layer can enlarge the dependency set of site $q$ only by one reverse step along an incoming edge $r\to q$ of $\Gamma_R$.  Hence after $L$ such layers, site $q$ can depend only on the radius-$L$ in-neighborhood of $q$.

If $d_R(p,q)>L$, the approximant's $q$-coordinate is independent of $x_p$.  Choose two points of $[0,1]^P$ that differ only in the $p$-coordinate, with values $0$ and $1$.  The approximant has the same $q$-output on both points, while the target outputs differ by $1$.  Therefore the uniform error is at least $1/2$ for one of the two points.

It remains to prove the diameter-depth compilation.  Let $Z$ be the equivariant bundle over the pair-state set $P\times P$ defined by
\[
Z_{(q,p)}:=V_p,
\qquad
T^Z_{\gamma,(q,p)}:=T^V_{\gamma,p}:Z_{(q,p)}\to Z_{(\gamma q,\gamma p)}.
\]
Equivalently, the carrier site $q$ stores the source-indexed memory $E_q=\bigoplus_{p\in P}V_p$, and $G$ sends the $p$-slot over $q$ to the $\gamma p$-slot over $\gamma q$.

First seed each source into its own labeled slot by the diagonal map $I:X\to Z$,
\[
(Ix)_{q,p}:=
\begin{cases}
 x_p, & q=p,\\
 0, & q\neq p.
\end{cases}
\]
This support is $G$-stable and carrier-local.  Next define the carrier-$R$-local propagation operator $M:Z\to Z$ by
\[
(Mz)_{q,p}:=\sum_{\substack{r\in P\\(q,r)\in R}}z_{r,p}.
\]
Its kernel blocks satisfy
\[
K((q,p),(r,s))=0
\qquad\text{unless } s=p \text{ and } (q,r)\in R,
\]
so $M$ is pair-state $R$-local in the sense of Definition~\ref{def:R-local-oenn}.  Since $R$ is $G$-invariant and the $p$-slot is transported to the $\gamma p$-slot by $T^V_{\gamma,p}$, the nonzero block kernels of $M$ also satisfy the auxiliary transporter law.  Hence $M$ is a pair-state $R$-local orbital linear layer.

For every $q,p\in P$,
\[
(M^D Ix)_{q,p}=a_{q,p}x_p,
\]
where $a_{q,p}$ is the number of directed length-$D$ walks from $p$ to $q$ in $\Gamma_R$.  Because $R$ contains the diagonal, walks can be padded by self-loops; since $D$ is the directed diameter and $\Gamma_R$ is strongly connected, $a_{q,p}>0$ for all $q,p$.  The $G$-invariance of $R$ gives
\[
a_{\gamma q,\gamma p}=a_{q,p}
\qquad(\gamma\in G).
\]
Therefore the pairwise diagonal rescaling $S:Z\to Z$ defined by
\[
(Sz)_{q,p}:=a_{q,p}^{-1}z_{q,p}
\]
satisfies the transporter law and is carrier-local.  Now $(SM^D Ix)_{q,p}=x_p$ for all $q,p$.

Finally collect the pair-state slots at each carrier site by the carrier-local equivariant map $C:Z\to\bigoplus_{q\in P}\widetilde X_q$,
\[
(Cz)_q=(z_{q,p})_{p\in P}\in\bigoplus_{p\in P}V_p=X.
\]
The transports on $Z$ and $\widetilde X$ make $C$ equivariant.  Hence $CSM^DI=B$, proving that the broadcast is exactly compiled by $D$ pair-state local inter-site propagation layers plus carrier-local seed, rescaling, and collection maps.

The universality assertion follows from the proof of Theorem~\ref{thm:OENN-UAT}, which constructs an approximant $F=N\circ B$, where $B$ is the broadcast layer and $N$ is pointwise.  The broadcast $B$ has just been exactly realized by pair-state $R$-local primitives, and $N$ is a pointwise Reynolds layer, hence uses no inter-site communication.  Therefore the same approximant belongs to $\mathrm{OENN}^{R\text{-pair}}_\sigma(X,Y)$.  If $L<D$, the definition of directed diameter gives sites $p,q$ with $d_R(p,q)>L$, so the lower-bound part gives the stated sharpness.
\end{proof}

\begin{proof}[Proof of Corollary~\ref{cor:cover-local-oenn-uat}]
The relation $R_{\mathrm{cov}}$ is $G$-invariant because $G$ acts by order automorphisms and therefore preserves cover relations.  Its communication graph is the self-looped directed version of the undirected Hasse graph with both orientations on every cover edge.  Connectedness of the undirected Hasse graph is therefore equivalent to strong connectedness of this communication graph.  Apply Theorem~\ref{thm:diameter-sharp-local-uat}.
\end{proof}

\subsection{Proofs for Section~\ref{sec:examples}}

\begin{proof}[Proof of Lemma~\ref{lem:sheaf-adjoint-equivariance}]
Choose one representative $p_0$ in each $G$-orbit of cells.  Average any inner product on $V_{p_0}$ over the stabilizer $G_{p_0}$, and transport this inner product to $V_{\gamma p_0}$ by declaring $T^V_{\gamma,p_0}$ to be an isometry.  The stabilizer invariance makes the definition independent of the chosen $\gamma$, and the cocycle identity for $T^V$ makes all transports isometries.

Taking Hilbert adjoints in \eqref{eq:sheaf-naturality} gives
\[
R^*_{p\to q}(T^V_{\gamma,q})^{-1}=(T^V_{\gamma,p})^{-1}R^*_{\gamma p\to\gamma q},
\]
because the transports are isometries.  Multiplying by $T^V_{\gamma,p}$ on the left and by $T^V_{\gamma,q}$ on the right gives \eqref{eq:sheaf-adjoint-naturality}.
\end{proof}

\begin{proof}[Proof of Lemma~\ref{lem:sheaf-orbit-transporter}]
The self terms are precisely the diagonal pair-orbit case of Proposition~\ref{prop:orbit-param}, so they satisfy \eqref{eq:transporter-law}.

For the upward terms, if $(q,p)=\alpha(q_{\calO},p_{\calO})$, then
\[
K^\uparrow(q,p)
=T^W_{\alpha,q_{\calO}}B^\uparrow_{\calO}R_{p_{\calO}\to q_{\calO}}
  (T^V_{\alpha,p_{\calO}})^{-1}.
\]
For any $\gamma\in G$, the kernel at $(\gamma q,\gamma p)$ is obtained by using $\gamma\alpha$ as transporter.  Hence, by functoriality,
\[
\begin{aligned}
K^\uparrow(\gamma q,\gamma p)T^V_{\gamma,p}
&=T^W_{\gamma\alpha,q_{\calO}}B^\uparrow_{\calO}R_{p_{\calO}\to q_{\calO}}
  (T^V_{\gamma\alpha,p_{\calO}})^{-1}T^V_{\gamma,p} \\
&=T^W_{\gamma,q}K^\uparrow(q,p).
\end{aligned}
\]
The downward calculation is the same, using the adjoint naturality \eqref{eq:sheaf-adjoint-naturality}.  In the globally shared notation below, this calculation reduces to
\[
B_\uparrow R_{\gamma p\to\gamma q}T^V_{\gamma,p}
=B_\uparrow T^V_{\gamma,q}R_{p\to q},
\]
which equals $T^W_{\gamma,q}B_\uparrow R_{p\to q}$ precisely when the shared block $B_\uparrow$ intertwines the relevant transports; this is automatic in the permutation-only case.  Therefore all kernels in \eqref{eq:sheaf-orbit-linear} satisfy the transporter law, and the layer is order-equivariant by Proposition~\ref{prop:transporter-law}.
\end{proof}

\end{document}